\documentclass{article}

\usepackage[accepted]{icml2019}
\usepackage{hyperref}
\usepackage{url}            
\usepackage{booktabs}       
\usepackage{amsfonts}       
\usepackage{nicefrac}       
\usepackage{microtype}      
\usepackage{url}
\usepackage{amsmath}
\usepackage{wrapfig}
\usepackage[inline]{enumitem}
\usepackage{bbm}
\usepackage{bm}
\usepackage{amsfonts}
\usepackage[table]{xcolor}
\usepackage{algorithm}
\usepackage{algorithmic}
\usepackage{booktabs}
\usepackage{caption}
\usepackage{subcaption}
\usepackage{graphicx}
\usepackage{subcaption}
\usepackage{xspace}
\usepackage{verbatim}
\usepackage{multirow}
\usepackage[percent]{overpic}
\usepackage{csquotes}
\usepackage{float}
\usepackage{rotating}

\newtheorem{theorem}{Theorem}[section]
\newtheorem{lemma}[theorem]{Lemma}
\newtheorem{claim}[theorem]{Claim}

\newtheorem{definition}[theorem]{Definition}
\newcommand{\algorithmname}[1]{\textnormal{\textsc{#1}}}
\newcommand{\leaf}{{\operatorname{leaf}}}
\newcommand{\Path}{{\operatorname{path}}}
\newcommand{\Top}{{\operatorname{top}}}
\newcommand{\Rand}{{\operatorname{rand}}}

\newcommand{\T}{{T}}

\newcommand{\parent}{{\operatorname{parent}}}

\newcommand{\rottext}[1]{\begin{sideways}#1\end{sideways}}
\newcommand{\sign}{{\operatorname{sign}}}
\newcommand{\Left}{{\operatorname{left}}}
\newcommand{\Right}{{\operatorname{right}}}
\newcommand{\Root}{{\operatorname{root}}}
\newcommand{\mem}{{\operatorname{mem}}}
\newcommand{\update}{{\operatorname{update}}}
\newcommand{\checkY}{\checkmark}
\newcommand{\checkN}{}
\newenvironment{proof}{\paragraph{Proof:}}{\hfill$\square$}
\newcommand{\DCMT}{CMT\xspace}
\newsavebox{\algleft}
\newsavebox{\algright}

\newcommand{\wen}[1]{#1}

\renewcommand{\sectionautorefname}{\S\kern-0.2em}      
\renewcommand{\subsectionautorefname}{\S\kern-0.2em}   
\renewcommand{\subsubsectionautorefname}{\S\kern-0.2em}

\icmltitlerunning{Contextual Memory Tree}

\begin{document}

\twocolumn[
\icmltitle{Contextual Memory Trees}




\begin{icmlauthorlist}
\icmlauthor{Wen Sun}{to}
\icmlauthor{Alina Beygelzimer}{goo}
\icmlauthor{Hal Daum\'e III}{ed}
\icmlauthor{John Langford}{ed}
\icmlauthor{Paul Mineiro}{edr}
\end{icmlauthorlist}

\icmlaffiliation{to}{Robotics Institute, Carnegie Mellon University, USA}
\icmlaffiliation{goo}{Yahoo! Research, New York, NY, USA}
\icmlaffiliation{ed}{Microsoft Research, New York, NY, USA}
\icmlaffiliation{edr}{Microsoft, USA}

\icmlcorrespondingauthor{Wen Sun}{wensun@cs.cmu.edu}

\icmlkeywords{Machine Learning, ICML}

\vskip 0.3in
]



\printAffiliationsAndNotice{}  

\begin{abstract}
We design and study a Contextual Memory Tree (\DCMT),
a learning \emph{memory controller} that inserts new memories into
an \emph{experience store} of unbounded size.
It is designed to efficiently query
for \emph{memories} from that store, supporting logarithmic time insertion and retrieval operations. Hence \DCMT can be integrated into existing statistical learning algorithms as an augmented memory unit without substantially increasing training and inference computation.  Furthermore \DCMT operates as a reduction to classification, allowing it to benefit from advances in representation or architecture.  We demonstrate the efficacy of \DCMT by augmenting existing multi-class and multi-label classification algorithms with \DCMT and observe statistical improvement. We also test \DCMT learning on several image-captioning tasks to demonstrate that it performs computationally better than a simple nearest neighbors memory system while benefitting from reward learning.\end{abstract}

\section{Introduction}
\label{sec:intro}


When a human makes a decision or answers a question, they are able to
do so while very quickly drawing upon a lifetime of remembered
experiences.  This ability to retrieve relevant experiences
efficiently from a memory store is currently lacking in most machine
learning systems (\autoref{sec:existing}).  We consider the problem of
learning an efficient online data structure for use as an external memory in
a reward-driven environment.  The key functionality of the Contextual
Memory Tree (CMT) data structure defined here is the ability to
\emph{insert} new memories into a learned key-value store, and to be
able to \emph{query} those memories in the future.  The storage and
query functionality in \DCMT is driven by an optional,
user-specified, external reward signal; it organizes memories so as to
maximize the downstream reward of queries.  In order to scale to very
large memories, our approach organizes memories in a tree structure,
guaranteeing logarithmic time (in the number of memories) operations
throughout (\autoref{sec:properties}).  Because \DCMT operates as
a reduction to classification, it does not prescribe a representation
for the keys and can leverage future advances in classification
techniques.

More formally, we define the data structure \DCMT (\autoref{sec:alg}),
which converts the problem of mapping \emph{queries} (keys) to
\emph{memories} (key-value pairs) into a collection of classification
problems.  
Experimentally (\autoref{sec:experiment}), we show this is useful
in three different settings.  \textbf{(a)} Few-shot learning in
extreme multiclass classification problems, where \DCMT is
used directly as a classifier (the queries are examples and the
values are class labels). \autoref{fig:few_shots_results} shows
that \emph{unsupervised} \DCMT{} can statistically outperform
other \emph{supervised} logarithmic-time baselines including
LOMTree \cite{choromanska2015logarithmic} and Recall Tree (RT)
\cite{daume2016logarithmic} with supervision providing further
improvement. \textbf{(b)} Extreme multi-label classification problems
where \DCMT is used to augment a One-Against-All (OAA) style inference
algorithm.  \textbf{(c)} Retrieval of images based on captions, where
\DCMT is used similarly to a nearest-neighbor retrieval system
(the queries are captions and the values are the corresponding
images).  External memories that persist across examples are also
potentially useful as inputs to downstream applications; for instance,
in natural language dialog tasks \cite{bartl17searchdialogue} and in
machine translation \cite{gu18searchnmt}, it can be useful to
retrieve similar past contexts (dialogs or documents) and augment the
input to the downstream system with these retrieved examples.
Memory-based systems can also be useful as a component of learned
reasoning systems \cite{Weston,DNC}.

\begin{figure*}[t]

\centering
\begin{subfigure}[l]{0.5\textwidth}
  \footnotesize
  \rowcolors{2}{gray!25}{white}
  \begin{tabular}{lccccc}
  & \rottext{Low Time} & \rottext{Small Space} & \rottext{Self-consistent} & \rottext{Incremental} & \rottext{Learning} \\
  \midrule
  Inverted Index  & \checkN & \checkY & \checkY & \checkN & \checkN \\
  Supervised Learning  & \checkY & \checkY & \checkN & \checkY & \checkY \\
  Nearest Neighbor  & \checkN & \checkY & \checkY & \checkY & \checkN \\
  Approx-NN  & \checkY & \checkY & \checkY & \checkN & \checkN \\
  Learned-NN  & \checkN & \checkY & \checkY & \checkN & \checkY \\
  Hashing  & \checkN & \checkY & \checkY & \checkN & \checkY \\
  Differentiable Memory  & \checkN & \checkY & \checkY & \checkY & \checkY \\
  \textbf{CMT}  & \checkY & \checkY & \checkY & \checkY & \checkY \\
\end{tabular}
\end{subfigure}\hfill
\begin{subfigure}[l]{0.5\textwidth}
  \includegraphics[width=\textwidth,keepaspectratio,trim={18 35 85 45},clip]{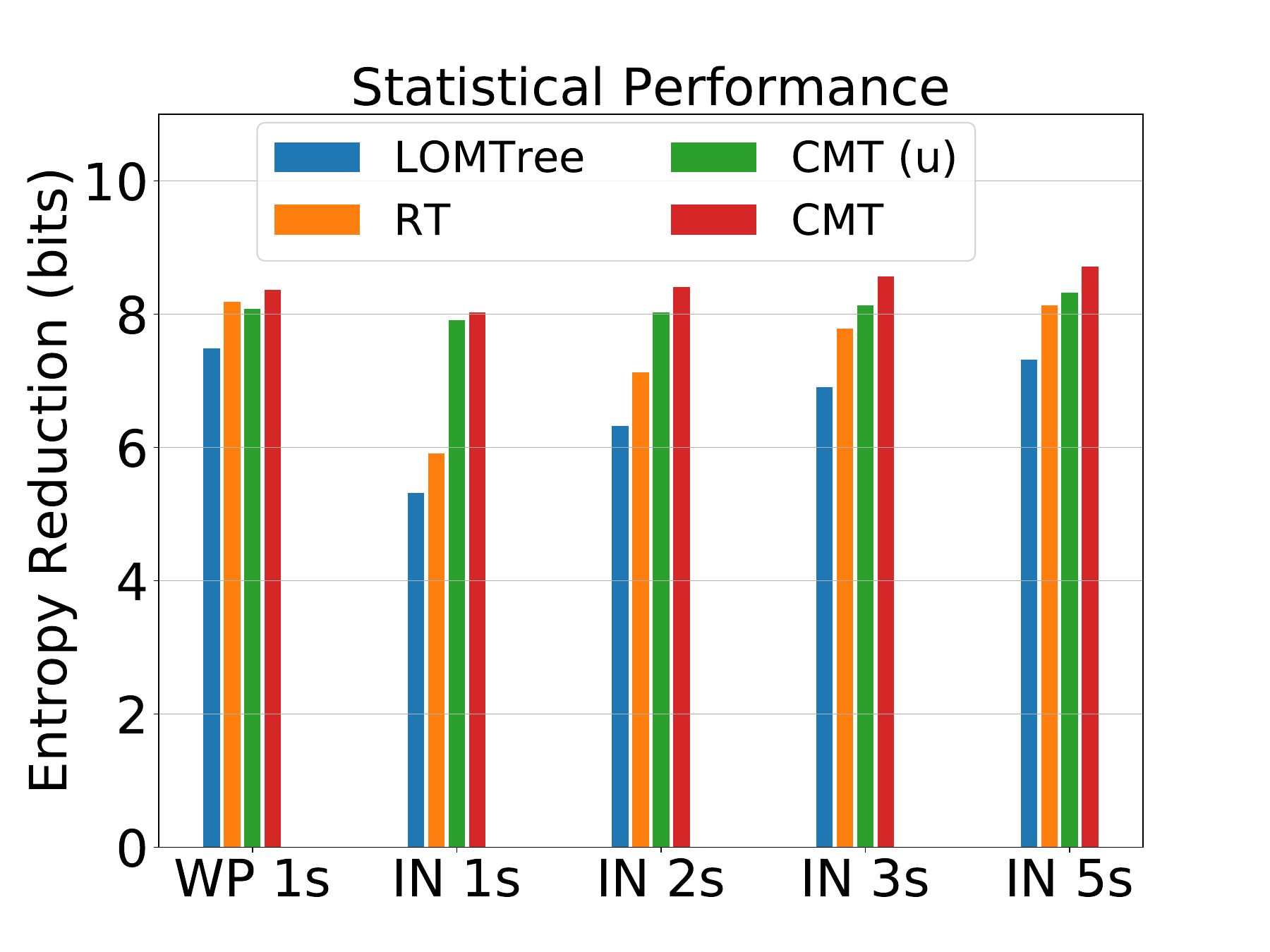}
\end{subfigure}
\caption{(Left) Desiderata satisfied by prior approaches; where answers vary with choices, we default towards 'yes'.~\label{table:alts}
(Right) Statistical performance (Entropy Reduction from the constant predictor.  Higher is better---see the experiments section.) on WikiPara one-shot  (WP 1s) dataset and ImageNet $S$-shot datasets (IN $S$s) \label{fig:few_shots_results} with baselines (LOMTree and RecallTree) and our proposed approach (unsupervised and supervised) \DCMT.}
\vspace{-10pt}
\end{figure*}

A memory $z = (x,\omega)$ is a pair of query $x$ and value $\omega$.  \DCMT operates in the following generic online manner, repeated over time:
\begin{enumerate}[nolistsep,noitemsep,leftmargin=2em]
\item Given a query $x$, retrieve $k$ associated memories $(u, \langle z_{1},z_{2},\dots, z_k\rangle) = \textsc{Query}(x)$ together with an identifier $u$.
\item If a reward $r_i$ for $z_i$ is observed, update the system via $\textsc{Update}( (x, z_i, r_i), u)$.
\item If a value $\omega$ associated with $x$ is available, $\textsc{Insert}$ a new memory $z = (x,  \omega)$ into the system.
\end{enumerate}

A natural goal in such a system is a notion of self-consistency.
If the system inserts $z = (x, \omega)$ into \DCMT,
then in subsequent rounds, one should expect that $(x,\omega)$ is retrieved when $\textsc{Query}(x)$ is issued again for the same $x$.  (For simplicity, we assume that all $x$ are unique.)
In order to achieve such self-consistency in a data structure that changes over time, we augment \DCMT with a ``Reroute'' operation, in which the data structure gradually reorganizes itself by removing old memories and re-inserting them on an amortized basis.
We find that this Reroute operation is essential to good empirical performance (\autoref{sec:ablation}).

\subsection{Existing Approaches}
\label{sec:existing}

The most standard associative memory system is a map data structure (e.g., hashmap, binary tree, relational database);
unfortunately, these do not \emph{generalize} across inputs---either an input is found exactly or it is not.
We are interested in memories that can {generalize} beyond exact lookups, and can \emph{learn} to do so based on past successes and failures
in an \emph{incremental}, online manner.
Because we wish to scale, the \emph{computation time} for all operations must be at most logarithmic in the number of memories,
with \emph{constant space overhead} per key-value pair.
Finally, as mentioned above, such a system should be \emph{self-consistent}.

There are many existing approaches beyond hashmaps, all of which miss
one of our desiderata (\autoref{table:alts}).  A basic approach for text
documents is an \textbf{inverted index}~\cite{Knuth,WAND}, which indexes
a document by the words that appear therein.  On the other end of the
spectrum, \textbf{supervised learning} can be viewed as remembering
(compiling) a large amount of experience into a predictor which may
offer very fast evaluation, but generally cannot explicitly query for
past memories (aka examples).

There has been substantial recent interest in coupling neural
networks with nearest neighbor variants.  Classical approaches are
inadequate: \begin{enumerate*}[label={\alph*)},font={\bfseries}]
\item \textbf{Exact nearest neighbor} algorithms (including memory
  systems that use them \cite{KaiserNRB17}) are computationally
  inefficient except in special cases \cite{RPTrees,CoverTree} and do
  not learn.  \item \textbf{Approximate Nearest Neighbors} via
  Locality-Sensitive Hashing~\cite{LSH} and MIPS~\cite{MIPS} address
  the problem of computational time, but not learning.  \item
  \textbf{Nearest Neighbors with Learned Metrics}~\cite{maxmarginnn}
  can learn, but are non-incremental.\end{enumerate*}

More recent results combine neural architectures with forms of
approximate nearest neighbor search to address these shortcomings.  For
example,~\cite{rae2016scaling} uses a representation learned for a
task with either randomized kd-trees or locality sensitive hashing on a
the Euclidean distance metric, both of which are
periodically recomputed.  The CMT instead learns at \emph{individual
  nodes} and works for \emph{any} representation, therefore, 
avoiding presupposing that a Euclidean metric is appropriate and could
potentially productively replace the approximate nearest neighbor
subsystem here.

Similarly,~\cite{chandar2016hierarchical} experiments with a variety
of K-MIPS (Maximum Inner Product Search) data structures which the
memory tree could potentially replace to create a higher ceiling on
performance in situations where MIPS is not the right notion of
similarity.

In~\cite{DBLP:journals/corr/AndrychowiczK16} the authors learn a
hierarchical data structure over a pre-partitioned set of memories with
a parameterized JOIN operator shared across nodes.  The use of
pre-partition makes the data structure particularly sensitive to the
(unspecified) order of that prepartition as discussed in appendix 6 of
the LOMTree~\cite{choromanska2015logarithmic}.  Furthermore, tieing the
parameters of JOIN across the nodes deeply constrains the
representation compared to our approach.

Many of these shortcomings are addressed by \textbf{learned hashing}-based models \cite{SemanticHash,PDBC}, which learn a hash function that works well at prediction time, but all current approaches are non-incremental and require substantial training-time overhead.
Finally, \textbf{differentiable memory systems}~\cite{Weston,DNC} are able to refine memories over time, but rely on gradient-descent-based techniques which incur a computational overhead that is inherently linear in the number of memories. 

There are works on leveraging memory systems to perform few-shot learning tasks (\cite{snell2017prototypical,strubell2017fast,santoro2016one}). However they are not logarithmic time and hence incapable of effectively operating at the large scales. Also note that they often address an easier version of the few-shot learning problem where training with a large number of labels for some classes is allowed as an initializer before the few-shot labels are observed. In contrast, we have no initialization phase. 
\section{The  Contextual Memory Tree}
\label{sec:alg}


At a high level, a \DCMT{} (\autoref{fig:cmt_tree}) is a near-balanced binary tree whose size dynamically increases as more memories are inserted. All memories are stored in leaf nodes with each leaf containing at most $c\log n$ memories, where $n$ is the total number of memories and $c$ is a constant independent of the number of memories.  

Learning happens at every node of \DCMT. Each internal node contains a learning router.  Given a query, \DCMT{} routes from the root to a leaf based on left-or-right decisions made by the routers along the way.  Each internal node optimizes a metric, which ensures both its router's ability to predict which sub-tree contains the best memory corresponding to the query, and the balance between its left and right subtrees.  \DCMT{} also contains a global learning scorer that predicts the reward of a memory for a query. The scorer is used at a leaf to decide which memories to return, with updates based on an external reward signal of memory quality. 



\subsection{Data Structures}

\begin{figure}[t]
  \centering
    \includegraphics[width=0.3\textwidth]{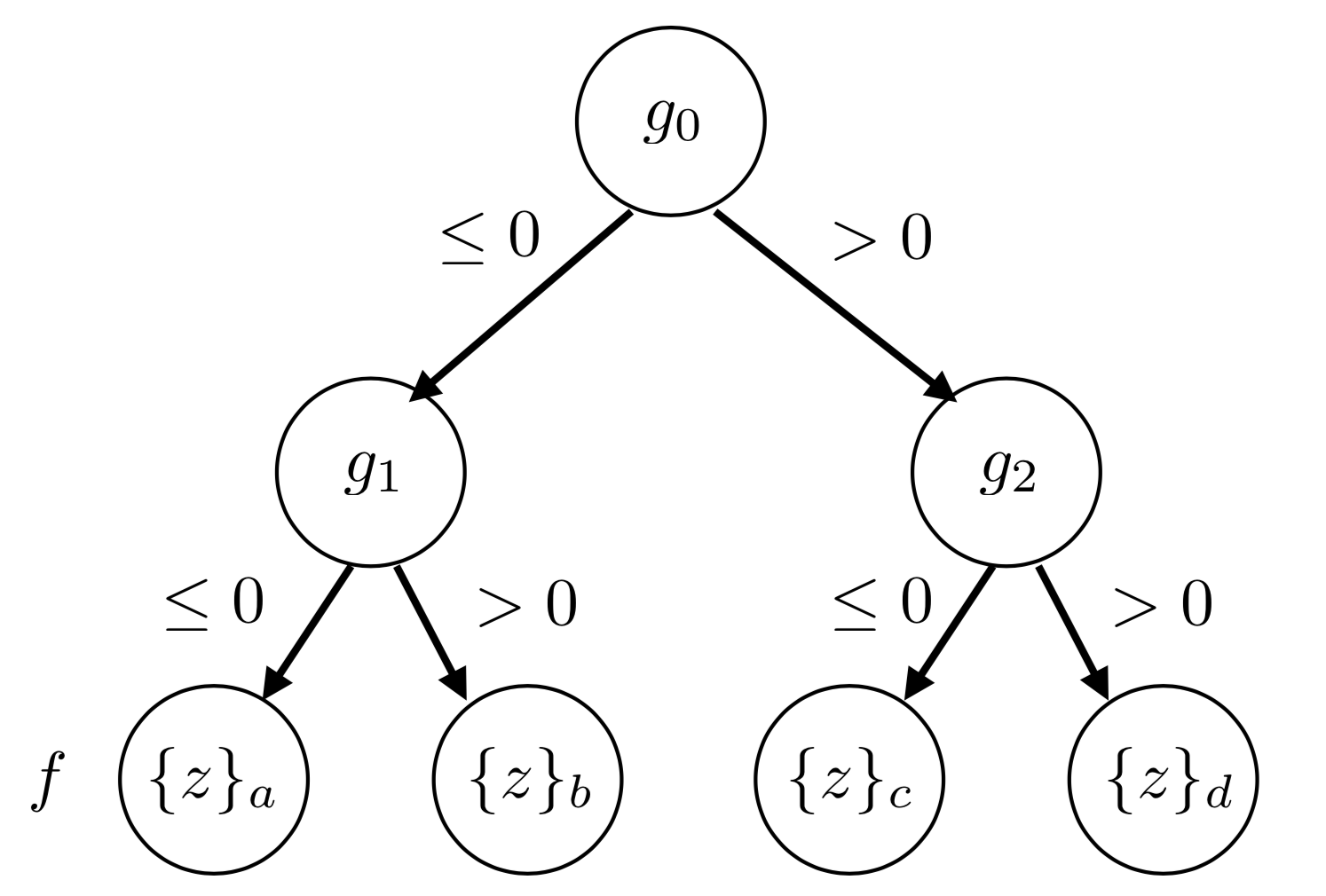}
    \caption{An example of \DCMT. Each internal node contains a binary classifier $g$ as the router, and every leaf stores a small set of memories $\{z\}$. All leafs share a learning scorer $f$ which computes a score of a memory $z$ and a query $x$, and is used to select a memory when a query reaches a leaf.}
    \label{fig:cmt_tree}
\end{figure}

\begin{figure*}
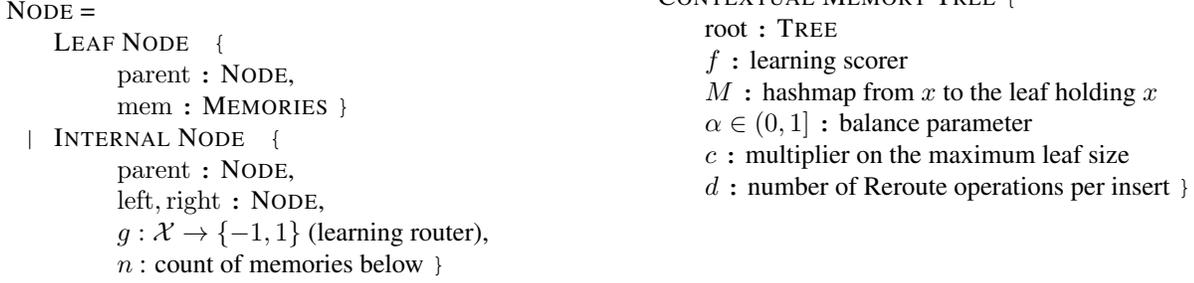

  \begin{tabular}{@{}l@{ }l@{}}
    \begin{minipage}{0.5\textwidth}
    \verb+ +\textsc{Node} = \\
    \verb+    +\textsc{Leaf Node} \verb+ {+\\
    \verb+        +$\parent$ \verb+:+ \textsc{Node}, \\
    \verb+        +$\mem$ \verb+:+ \textsc{Memories} \verb+}+ \\
    \verb+  | +\textsc{Internal Node} \verb+ {+\\
    \verb+        +$\parent$ \verb+:+ \textsc{Node}, \\
    \verb+        +$\Left,\Right$ \verb+:+ \textsc{Node}, \\
    \verb+        +$g : \mathcal{X} \rightarrow \{-1,1\}$ (learning router), \\
    \verb+        +$n$ : count of memories below \verb+}+
    \end{minipage}
    &
    \begin{minipage}{0.5\textwidth}
    \verb+ +\textsc{Contextual Memory Tree} \verb+{+ \\
    \verb+    +root \verb+:+ \textsc{Tree} \\
    \verb+    +$f$ \verb+:+ learning scorer \\
    \verb+    +$M$ \verb+:+ hashmap from $x$ to the leaf holding $x$\\
    \verb+    +$\alpha\in (0,1]$ \verb+:+ balance parameter\\
    \verb+    +$c$ \verb+:+ multiplier on the maximum leaf size\\
    \verb+    +$d$ \verb+:+ number of Reroute operations per insert  \verb+}+\\
  ~\\
  ~\\
    \end{minipage}
  \end{tabular}
   \caption{Data structures for internal and leaf nodes (left) and the full contextual memory tree (right). \label{fig:ds}}
\end{figure*}

A \emph{memory} consists of a \emph{query} (key) $x\in\mathcal{X}$ and its
associated \emph{value} $\omega\in{\Omega}$. We use $z$ to denote the memory
pair $(x,\omega)$ and define $\mathcal{Z} = \mathcal{X}\times{\Omega}$
as the set of $z$.  Given a memory $z$, we use $z.x$ and $z.\omega$ to represent the query and the value of $z$ respectively. 
For instance, for multiclass classification, $x$
is a feature vector and $\omega$ is a label.  Our memory store is organized
into a binary tree.  A \emph{leaf} node in \autoref{fig:ds} (left top)
consists of a parent and a set of memories.  Leaf nodes are connected by
internal nodes as in \autoref{fig:ds} (left, bottom).  An internal
node has a parent and two children, which may be either leaf or
internal nodes, a count $n$ of the number of memories beneath the
node, and a learning router $g: \mathcal{X}\to \{-1,1\}$ which both
routes via $g(x)$ and updates via $g.\update(x,y)$ for $y \in
\{-1,1\}$, or $g.\update(x,y,i)$ where $i \in\mathbb{R}^+$ is an
importance weight of $(x,y)$.  If $g(x)\leq 0$, we route $x$ left, and
otherwise right.


The contextual memory tree data structure in \autoref{fig:ds} (right) has a root
node, a parameter $\alpha \in [0,1]$ which controls how balanced the
tree is, a multiplier $c$ on the maximum number of memories stored in
any single leaf node, and a learning scorer $f:\mathcal{X}\times\mathcal{Z}\to \mathbb{R}$.  Given a query $x$ and memory $z$, the learning scorer predicts the reward one would receive if $z$ is returned as the retrieved memory for query $x$ via $f(x,z)$. Once a reward $r\in [0,1]$ is received for a pair of memory $z$ and query $x$, the learning scorer updates via
$f.\text{update}(x,z,r)$ to improve its ability to predict reward.  Finally,
the map $M$ maps examples to the leaf that contains them, making removal easy.

Given any internal node $v$ and query $x$, we define a data structure
$\Path$ representing the path taken from $v$ to a leaf: $\Path =
\{(v_i, a_i, p_i)\}_{i=1}^N$, where $v_1 = v$, $a_i\in \{\Left,\Right\}$
is the left or right decision made at $v_i$, $p_i\in [0,1]$ is the
probability with which $a_i$ was chosen.  As we show later, $\Path$
communicates to the update rule the information needed to create an
unbiased update of routers.

\subsection{Algorithms}
  
      \begin{algorithm}[t]
\caption{\algorithmname{Path}(query $x$, node $v$)}
\label{alg:path}
\begin{algorithmic}[1]
  \STATE $\Path \gets \emptyset$
  \WHILE{$v$ is not a leaf}
  \STATE $a \gets$ \textbf{if} $v.g(x) > 0$ \textbf{then} $\Right$ \textbf{else} $\Left$
  \STATE Add $(v,a,1)$ to $\Path$
  \STATE $v\gets v.a$
  \ENDWHILE
  \STATE $\Path.\leaf \gets v$
  \STATE \textbf{return} $\Path$
\end{algorithmic}
\end{algorithm}

\begin{algorithm}[t]
\caption{\algorithmname{Query}(query $x$, items to return $k$, exploration probability $\epsilon$)}
\label{alg:query}
\begin{algorithmic}[1]
  \STATE $\Path \gets \algorithmname{Path}(x,\Root)$, $\Path = \{(v_i,a_i,p_i)\}_{i=1}^N$ 
  \STATE $q\in_U [0,1]$
  \IF{$q \geq \epsilon$}
  	\STATE \textbf{return} $(\emptyset, \Top_k (\Path.\leaf,x))$
  \ELSE
  	\STATE Pick $i \in_U \{1,\ldots,N+1\}$
  	\IF {$i \leq N$}
  		\STATE Pick $a' \in_U \{\Right,\Left\}$
		\STATE $l = \algorithmname{Path}(x,v_i.a').\leaf$
		\STATE \textbf{return} 
$
((v_i,a',1/2), \Top_k (l,x))
$
	\ELSE
		\STATE  \textbf{return} 
$
	((\Path.\leaf,\perp,\perp), \Rand_k(\Path.\leaf,x) )
$
 	\ENDIF
  \ENDIF
\end{algorithmic}
\end{algorithm}

\begin{algorithm}[t]
\caption{\algorithmname{Update}$\left( (x, z, r), (v,a,p) \right)$}
\label{alg:update}
\begin{algorithmic}[1]
        \IF {$v$ is a leaf}
	    \STATE $f.\update(x, z, r)$
       \ELSE
  		\STATE{$\hat{r} \gets \frac{r}{p}(\mathbf{1}(a=\Right)-\mathbf{1}(a=\Left))$}
	  	\STATE{$y \gets (1-\alpha) \hat{r} + \alpha (\log v.\Left.n - \log v.\Right.n)$}
	 	\STATE{$v.g.\update(x, \sign(y), |y|)$}
        \ENDIF
 \STATE Run \algorithmname{Reroute} $d$ times
\end{algorithmic}
\end{algorithm}


All algorithms work given a contextual memory tree $\T$. For brievity, we drop $T$ when referencing its fields.
We use $\in_U P$ to chose uniformly at random from a set $P$.

Algorithm\autoref{alg:path} (\algorithmname{Path}) routes a query $x$ from any node $v$ to a leaf, returning the $\Path$ traversed.

Algorithm\autoref{alg:query} (\algorithmname{Query}) takes a query $x$ as input and returns at most $k$ memories. The parameter $\epsilon \in [0,1]$ determines the probability of exploration during training.  Algorithm\autoref{alg:query} first deterministically routes the query $x$ to a leaf and records the path traversed, $\Path$. With probability $1-\epsilon$, we simply return the best memories stored in $\Path.\leaf$:  For a query $x$ and leaf $l$, we use $\Top_k(l,x)$ as a shorthand for the set of $\min\{k,\vert l.\mem\vert\}$ memories $z$ in $l.\mem$ with the largest $f(x,z)$, breaking ties randomly.  We also use $\Rand_k(l,x)$ for a random subset of $\min\{k,\vert l.\mem\vert\}$ memories in $l.\mem$.

With the remaining probability $\epsilon$, we uniformly sample a node along $\Path$ including $\Path.\leaf$.  If we sampled an internal node $v$, we choose a random action $a'$ and call $\Path(x, v.a')$ to route $x$ to a leaf.  This exploration gives us a chance to discover potentially better memories stored in the other subtrees beneath $v$, which allows us to improve the quality of the router at node $v$.  We do uniform exploration at a uniformly chosen node but other schemes are possible.
If we sampled $\Path.\leaf$, we return a random set of memories stored in the leaf, in order to update and improve the learning scorer $f$.
The shorter the path, the higher the probability that exploration happens at the leaf.



\begin{algorithm}[H]
\caption{\algorithmname{Insert}$(\text{node } v,\text{memory } z,\text{Reroute } d)$}
\label{alg:insert}
\begin{algorithmic}[1]
  \WHILE{$v$ is not a leaf}
  \STATE $B = \log v.\Left.n - \log v.\Right.n$
  \STATE{$y \gets \text{sign}\left((1-\alpha) v.g(z.x) + \alpha B \right)$}
  \STATE{$v.g.\text{update}(z.x, y)$}
  \STATE{$v.n \gets v.n+1$}
  \STATE{$v \gets $ \textbf{if} $v.g(z.x) > 0$ \textbf{then} $v.\text{right}$ \textbf{else} $v.\text{left}$}
  \ENDWHILE
  \STATE \algorithmname{InsertLeaf}($v$, $z$)
  \STATE Run \algorithmname{Reroute} $d$ times
\end{algorithmic}
\end{algorithm}


After a query for $x$, we may receive a reward $r$ for a returned memory $z$.  In this case, Algorithm\autoref{alg:update} (\algorithmname{Update}) uses the first triple returned by \algorithmname{Query} to update the router making a randomized decision.  More precisely, Algorithm~\autoref{alg:update} computes an unbiased estimate of the reward difference of the left/right decision which is then mixed with a balance-inducing term on line 5.  When randomization occurred at the leaf, the scorer $f$ is updated instead.

The \algorithmname{Insert} operation is given 
in Algorithm\autoref{alg:insert}. It routes the memory $z$ to be inserted according to the
query $z.x$ from the root to a leaf using internal learning routers,
updating them on descent.  Once reaching a leaf node, $z$
is added into that leaf via \algorithmname{InsertLeaf}.
The label
definition on line 3 in \algorithmname{Insert} is the same as was used
in~\cite{beygelzimer2009conditional}.  That use, however, was for a
different problem (conditional label estimation) and is applied differently 
(controlling the routing of examples rather than just advising a
learning algorithm).  As a consequence, the proofs of correctness
given in section~\ref{sec:computation} differ.

When the number of memories stored in any leaf exceeds the log of the total number of memories, a leaf is split according to Algorithm\autoref{alg:split_leaf}  (\algorithmname{InsertLeaf}).  The leaf node $v$ is promoted to an internal node with two leaf children and a binary classifier $g$ with all memories inserted at $v$.

Because updates are online, they may result in a lack of
self-consistency for previous insertions. This is fixed by \algorithmname{Reroute} (Algorithm \autoref{alg:dream}) on an amortized basis.
Specifically, after every \algorithmname{Insert} operation we call \algorithmname{Reroute}, which randomly samples an example from all the examples, extracts the sampled example from the tree, and then re-inserts it.
This relies on the \algorithmname{Remove} (Algorithm \autoref{alg:remove}) operation, which  finds the location of a memory using the hashmap
then ascends to the parent cleaning up accounting.  When a leaf
node has zero memories, it is removed.  

\begin{algorithm}[t]
\caption{\algorithmname{InsertLeaf}$(\text{leaf node } v, \text{memory } z)$}
\label{alg:split_leaf}
\begin{algorithmic}[1]
  \STATE $v.\mem \gets v.\mem \cup \{ z\}$
  \IF {$|v.\mem| > c \log (\text{root}.n)$}
  \STATE $v' \gets $ a new internal node with two new children
  \FOR{$z\in v.\mem$}
  \STATE \algorithmname{Insert}$(v', z, 0)$
  \ENDFOR
  \STATE $v \leftarrow v'$
  \ENDIF
\end{algorithmic}
\end{algorithm}

\begin{algorithm}[t]
\caption{\algorithmname{Remove}$(x)$}
\label{alg:remove}
\begin{algorithmic}[1]
  \STATE Find $v \gets M(x)$, leaf containing $x$
  \STATE $v.\mem \gets v.\mem \setminus \{x\}$ 
  \WHILE {$v$ is not $\text{root}$}
  \IF {$v.n > 0$}
  \STATE $v.n \gets v.n-1$
  \STATE $v \gets v.\text{parent}$
  \ELSE 
  \STATE $v' = $ the other child of $v.\text{parent}$.
  \STATE $v.\text{parent} \gets v'$
  \STATE $v \gets v'$ 
  \ENDIF
  \ENDWHILE
\end{algorithmic}
\end{algorithm} 

\begin{algorithm}[t]
\caption{\algorithmname{Reroute}$()$}
\label{alg:dream}
\begin{algorithmic}[1]
\STATE Sample $z \in_U M$
\STATE \algorithmname{Remove}$(z.x)$
\STATE \algorithmname{Insert}$(\Root, z, 0)$
\end{algorithmic}
\end{algorithm} 


\section{Properties}
\label{sec:properties}
There are five properties that we want \DCMT to satisfy simultaneously (see \autoref{table:alts} (left) for the five properties).
Storage (in appendix~\ref{sec:storage}) and Incrementality (in
appendix~\ref{sec:incrementality}) are easy observations.

Appendix~\ref{sec:self} shows that in the limit of many \algorithmname{Reroute}s, self-consistency (defined below) is achieved.
\begin{definition} A {\normalfont \DCMT} is {\bf self-consistent} if for all $z$ with a unique $z.x$, $z=$\algorithmname{Query}$(z.x,1,0)$.  \end{definition}

Appendix~\ref{sec:learning} shows a learning property: 
Every internal router asymptotically optimizes to a local maxima of an objective function that mirrors line 5 of \algorithmname{Update}.


This leaves only logarithmic computational time, which we address next.

\subsection{Computational Time}\label{sec:computation}
The computational time analysis naturally breaks into two parts, 
partition quality at the nodes and the time complexity given good partitions. 
 To
connect the two, we first define partition quality.

\begin{definition} A {\bf $K$-balanced partition} of any set has each element of the partition containing at least a $1/K$ fraction of the original set.
  \end{definition}

When partitioning into two sets, $K \geq 2$ is required.  Smaller $K$
result in smaller computational complexities at the cost of worse
predictive performance in practice.

Define the progressive training error of a learning router $g$ after
seeing $T$ examples $x_1,\ldots,x_T$ as $p =
\frac{1}{T}\sum_{t=1}^T\mathbf{1}[g(x_t)\neq y_t]$, where $y_t$ is the
label assigned in line 3 of \algorithmname{Insert}, and $g(x_t)$ is
evaluated immediately after calling $g.\update(x_t,y_t)$ so a mistake
occurs when $g(x_t)$ disagrees with $y_t$ after the update.  The next
theorem proves a bound on the partition balance dependent on the
progressive training error of a node's router and $\alpha$.

\begin{theorem}\label{thm:partition}(Partition bound)
At any point, a router with a progressive training error of $p$ creates a 
$\frac{1+\exp(\frac{1-\alpha}{\alpha})}{(1-p) - \left(1+\exp(\frac{1-\alpha}{\alpha})\right)\frac{1}{T}}$-balanced partition.
\end{theorem}
The proof is in appendix~\ref{proof:partition}, followed by a bound on the depth of $K$-partition trees in appendix~\ref{proof:depth}.  
As long as $(1-p) > \exp(\frac{1-\alpha}{\alpha})\frac{1}{T}$ holds, \autoref{thm:partition} provides a nontrivial bound on partition.
Examining limits, when $p = 0$, $\alpha = 1$ and $T = \infty$, we have $K = 2$, which means \DCMT{} becomes a perfectly balanced binary tree.   If $p = 0.5$ (e.g., $g$ guesses at random), $\alpha = 0.9$ (used in all our experiments) and $T = \infty$, we have $K \leq 4.3$. For any fixed $T$, a smaller progressive error  $p$ and a larger $\alpha$ lead to a smaller $K$.



Next, we prove that $K$ controls the computational time.

\begin{theorem}\label{thm:computation}(Computational Time) If every router $g$ in a \DCMT with $T$ previous calls to \algorithmname{Insert} creates a $K$-partition, the worst case computation is $O(d (K + c) \log T)$ for \algorithmname{Insert},
  $O((K+c) \log T)$ for \algorithmname{Query}, and $O(1)$ for \algorithmname{Update} if all stated
  operations are atomic.
 \end{theorem}
The proof is in appendix~\ref{proof:computation}.  This theorem
establishes logarithmic time computation given that $K$-partitions are
created.  These two theorems imply that the computation is logarithmic
time for all learning algorithms achieving a \emph{training} error
significantly better than $1$.

\newcolumntype{g}{>{\columncolor{black!10}}c}

\section{Experiments}
\label{sec:experiment}
\DCMT is a subsystem for other learning tasks, so it assists other inference and learning algorithms.  We test the application of \DCMT to three systems, for multiclass classification, multilabel classification, and image retrieval.  
Seperately, we also ablate various elements of \DCMT to discover its strengths and weaknesses.

We implemented \DCMT as a reduction to Vowpal
Wabbit's \cite{langford2007vowpal} default learning algorithm. \footnote{\url{https://github.com/LAIRLAB/vowpal_wabbit/tree/master/demo/memory_tree}} The routers ($g$) and the learning scorer ($f$) are all linear functions and are incrementally updated by an Adagrad \citep{duchi2011adaptive}  gradient method in VW.  
Similarly, most baselines are implemented in the
same system with a similar or higher level of optimization.

\subsection{Application: Online Extreme Multi-Class Classification}

Since \DCMT operates online, we can evaluate its online performance using progressive validation \cite{blum1999beating} (i.e., testing each example ahead of training).  Used online, we \textsc{Query} for an example, evaluate its loss, then apply \textsc{Update} with the observed loss followed by \textsc{Insert} of the data point.  In a multiclass classification setting, a memory $z$ is a feature vector $x$ and label $\omega$. Given a query $x$, \DCMT returns a memory $z$ and receives a reward signal $\bm{1}[z.\omega = \omega]$ for update. Finally, \DCMT inserts $(x,\omega)$.

We test the online learning ability of \DCMT on two multiclass classification datasets, ALOI (1000 labels with 100 examples per label) and WikiPara 3-shot (10000 labels with 3 examples per label), against two other logarithmic-time online multiclass classification algorithms, LOMTree \cite{choromanska2015logarithmic} and Recall Tree \cite{daume2016logarithmic}.  We also compare against a linear-time online multiclass classification algorithm, One-Against-All (OAA).

 \begin{figure}[t]
    \begin{subfigure}[b]{0.235\textwidth}
         \includegraphics[width=1.135\textwidth,keepaspectratio]{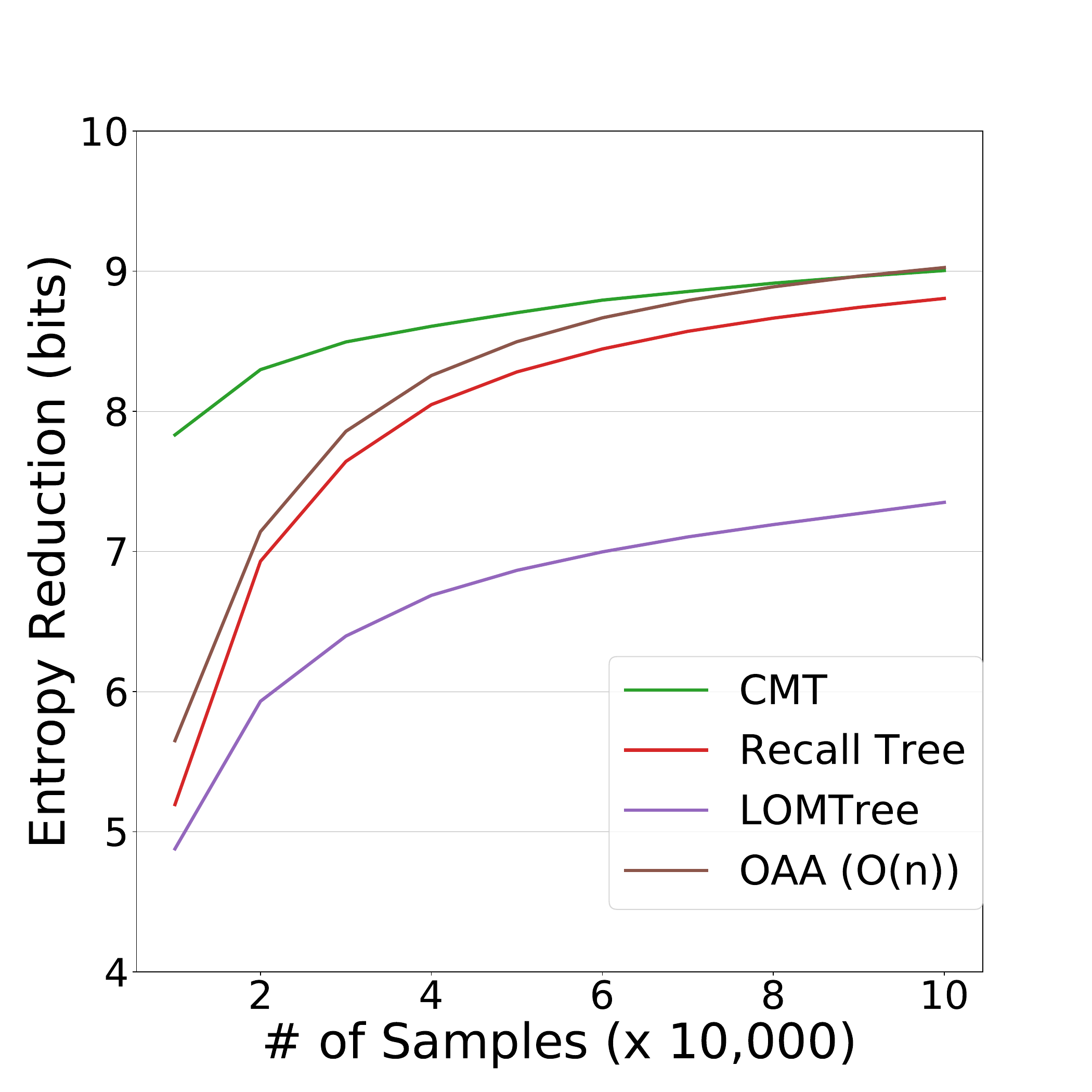}
         \caption{ALOI}
         \label{fig:aloi_online}
     \end{subfigure}
      \begin{subfigure}[b]{0.235\textwidth}
     	\includegraphics[width=1.135\textwidth,keepaspectratio]{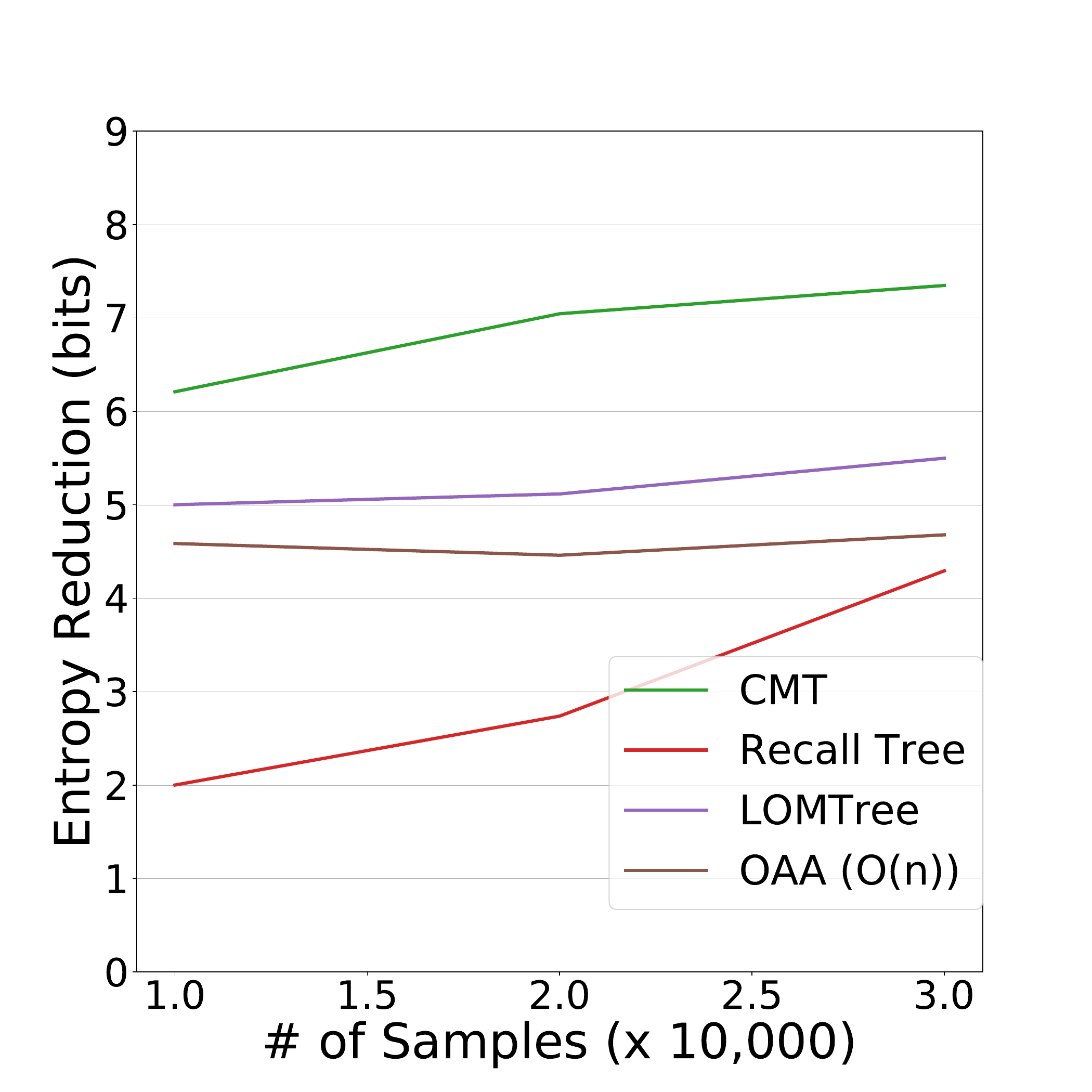}
         \caption{WikiPara 3-shot}
         \label{fig:wiki_online}
      \end{subfigure}
     \caption{(a) Online progressive performance of \DCMT with respect to the number of samples  on ALOI (a) and WikiPara 3-shot (b). \DCMT{} consistently outperforms all baselines.}
 \label{fig:online_perform}
 \vspace{-10pt}
 \end{figure}

\autoref{fig:online_perform} summarizes the results in terms of progressive performance.  On both datasets, we report entropy reduction from the constant predictor (the higher the better).  The entropy reduction of a predictor $A$ from another predictor $B$ is defined as $\log_2( p_{A}) - \log_2(p_{B})$, where $p_A$ and $p_B$ are prediction accuracies of $A$ and $B$.

{\bf Conclusion}: \DCMT greatly outperforms the baselines in the small
number of examples per label regime.  This appears to be primarily due
to the value of explicit memories over learned parameters in this regime. 

\begin{table*}[t]
  \centering
  \scriptsize
  \begin{tabular}{lcccgggccc}
    \toprule
    & \multicolumn{3}{c}{\textbf{RCV1-1K}}
    & \multicolumn{3}{g}{\textbf{AmazonCat-13K}}
    & \multicolumn{3}{c}{\textbf{Wiki10-31K}} \\
    \textbf{Approach} & loss & Test time & Train time & loss & Test time & Train time & loss & Test time & Train time \\
    \midrule
    \DCMT & $2.5$ & $1.4$ms & $1.9$hr & $3.2$ & $1.7$ms & \phantom{$1$}$5.3$hr & $18.3$& \phantom{$1$}$25.3$ms & $1.3$hr \\
    OAA & $2.6$ & $0.5$ms & $1.3$hr & $3.0$ & $8.7$ms & $15.5$hr & $20.3$ & $327.1$ms & $7.2$hr\\
    \bottomrule
  \end{tabular}
\vspace{-5pt}
\caption{Hamming Loss, test time per example (ms), and training time (hr) for multi-label tasks.}
\label{tab:multi_label_perf}
\vspace{-10pt}
\end{table*}

\subsection{Application: Batch Few-shot Multi-Class Classification}


We can also use \DCMT{} in an offline testing mode as well by using \DCMT{} with multiple passes over the the training dataset and testing it on a separate test set.  We again use \DCMT on few-shot multi-class classification, comparing it to LOMTree and Recall Tree. 

Starting first with the ALOI dataset, we tested both the unsupervised version (i.e., using only \textsc{Insert}) and the supervised version (i.e., using \textsc{Insert} for the first pass, and using \textsc{Update} for subsequent passes). We used three passes for all algorithms. The supervised version of \DCMT achieved 26.3\% test prediction error, outperforming LOMTree (66.7\%) and Recall Tree (28.8\%). The supervised version of \DCMT also significantly outperforms the unsupervised one (75.8\% error rate), showing the benefit of the \textsc{Update} procedure. Since ALOI has 1000 classes, a constant predictor has prediction error larger than 99\%.


We then test \DCMT on more challenging few-shot multi-class classification datasets, WikiPara $S$-shot ($S=1,2,3$) and ImageNet $S$-shot ($S=1,2,3,5$) with only $S$ examples per label.  \autoref{fig:few_shots_results} summarizes the statistical performance (entropy reduction compared to a constant predictor) of supervised \DCMT, unsupervised \DCMT (denoted as \DCMT(u)), and the two logarithmic-time baselines. For one-shot experiments (WP 1-s and IN 1-s on \autoref{fig:few_shots_results}), \DCMT  outperforms all baselines. The edge of \DCMT degrades gradually over baselines as $S$ increases (IN $S$s with  $S>1$ in \autoref{fig:few_shots_results}).  All details are included in \autoref{tab:details_few_shot_table} in Appendix~\autoref{sec:few_shot_details}. 

{\bf Conclusion}: The high performance of \DCMT with a small number of examples per label persists in batch training.  The remarkable performance of unsupervised \DCMT over supervised baselines suggests self-consistency can provide  
nearest-neighbor performance without explicit reward.

\subsection{Application: Multi-Label Classification with an External Inference Algorithm}

In this set of experiments, instead of using \DCMT as
an inference algorithm, we integrate \DCMT with an external inference
procedure based on One-Against-All.  \DCMT is not aware of the
external multi-label classification task, so this is an example of how an
external inference algorithm can leverage the returned memories as an
extra source of information to improve performance. Here each memory
$z$ consists of a feature vector $x$ and label vector $\omega\in
\{0,1\}^M$, where $M$ is the number of unique labels.  Given a query
$x$, its ground truth label vector $\omega$, and a memory $z$, we choose
the F1-score between $\omega$ and $z.\omega$ as the reward signal.  We set $k$
to $c\log(N)$ (i.e., \DCMT returns all memories in the leaf we
reach).  Given a query $x$, with the returned memories
$\{z_1,...,z_k\}$, the external inference procedure extracts the
unique labels from the returned memories and performs a
One-Against-Some (OAS) inference \cite{daume2016logarithmic} using the
extracted labels.\footnote{OAS takes $x$ and a small
  set of candidate labels and returns the labels with a positive
  score, according to a learned scoring function. After
  prediction, the OAS predictor receives the true labels $y$
  associated with this $x$ and performs an update to its score
  function based on the true labels and the small candidate label
  set.}  The external system then calls \algorithmname{Update} for the
returned memories.  Since \DCMT returns logarithmically many memories, we
guarantee that the number of unique labels from the returned memories
is also logarithmic. Hence augmenting OAS with \DCMT enables
logarithmic inference and training time.

We compare \DCMT-augmented OAS with multi-label OAA under the Hamming
loss. We compare \DCMT-augmented OAS to OAA on three multi-label
datasets, RCV1-1K \cite{prabhu2014fastxml}, AmazonCat-13K
\cite{mcauley2013hidden}, and Wiki-31K
\cite{zubiaga2012enhancing,bhatia2015sparse}.
(The datasets are described in \autoref{tab:dataset_multi_label} in
Appendix~\ref{sec:datasets}.)  \autoref{tab:multi_label_perf} summarizes the performance
of \DCMT and OAA. (LOMTree and Recall Tree are excluded because they do not operate in multi-label settings.)

{\bf Conclusion}: \DCMT-augmented OAS achieves similar statistical
performance to OAA, even mildly outperforming OAA on Wiki10-31K, while
gaining significant computational speed up over a vector optimized OAA in training and inference
on datasets with a large number of labels (e.g., AmazonCat-13K and
Wiki10-31K). Note that the VW implementation of OAA operates at a higher level optimization  and involves vectorized computations that increase throughput by a factor of 10 to 20. Hence we observe for RCV1-1K with 1K labels, OAA can actually be more computationally efficient then \DCMT. This set of experiments shows that
\DCMT-augmented OAS can win over OAA both statistically and
computationally for challenging few-shot multi-label datasets with a
large number of labels.

\subsection{Application: Image Retrieval}

We test \DCMT on an image retrieval task where the goal is to find an image given a caption.  We used three benchmark datasets, (1) UIUC Pascal Dataset \cite{rashtchian2010collecting}, (2) Flickr8k dataset \cite{hodosh2013framing}, and (3) MS COCO \cite{lin2014microsoft}, with feature representations described in \autoref{sec:datasets}.  
Here, a memory $z$ consists of (features of) a caption $x$ and an image $\omega$.  Given a query, \DCMT returns a memory $z = (x,\omega)$.
Our reward function is the cosine similarity between the returned memory's image $z.\omega$, and the ground truth image $\omega$ associated with the query $x$.


To show the benefit of learning in \DCMT, we compare it to Nearest Neighbor Search (NNS) and a KD-Tree as an Approximate NN data structure on this task, using the Euclidean distance  $\|x - z.x \|_2$ in the feature space of captions as the NNS metric.  
Both \DCMT and NNS are tested on a separate test set, with the average reward of the retrieved memory reported.

\autoref{tab:image_experiment_small} summarizes the speedup over NNS (implemented using a linear scan) and \wen{KD-Tree (KD tree implementation from scikit-learn \cite{pedregosa2011scikit}).  Note that in our datasets, the feature of a query is high dimensional ($2^{20}$) but extremely sparse. Since KD-Tree cannot take advantage of  sparsity, both the construction  and inference procedure is extremely slow (even slower than a NNS). We also emphasize here that a KD-Tree does not operate in an online manner. Hence in our experiments, we have to feed all queries from the entire training dataset to KD-Tree to initialize its construction, which makes it impossible to initialize the run of KD-Tree on MSCOCO. } 

{\bf Conclusion}: The difference in reward is negligible (on the order of $10^{-3}$) and statistically insignificant. (See Appendix \autoref{tab:image} for details.)
However, \DCMT is significantly faster.

\begin{table}[H]
  \vspace{-3mm}
  \centering
  \footnotesize
  \begin{tabular}{lcc}
    \toprule
    & \multicolumn{2}{c}{\DCMT} \\
    & unsup & sup \\
    \midrule
    \textbf{Pascal}   & 5.7 / 9400 & 1.3 / 2100 \\
    \textbf{Flickr8k} & 26.0 / 33000 & 6.0 / 7700 \\
    \textbf{MSCOCO}   & 21.0 / $\sim$ & 6.5 / $\sim$ \\
    \bottomrule
  \end{tabular}
  \caption{\footnotesize Speedups over linear NNS (left) and KD-Tree (right), in (unsup)ervised and (sup)ervised mode.
    \label{tab:image_experiment_small}}
\vspace{-10pt}
\end{table}


\subsection{Ablation Analysis of \DCMT}
\label{sec:ablation}

We conduct experiments to perform an ablation study of \DCMT in the context of multi-class classification, where it operates directly as an inference algorithm.

 \begin{figure}[t]
  	\centering
    \begin{subfigure}[b]{0.22\textwidth}
         \includegraphics[width=1.12\textwidth,keepaspectratio]{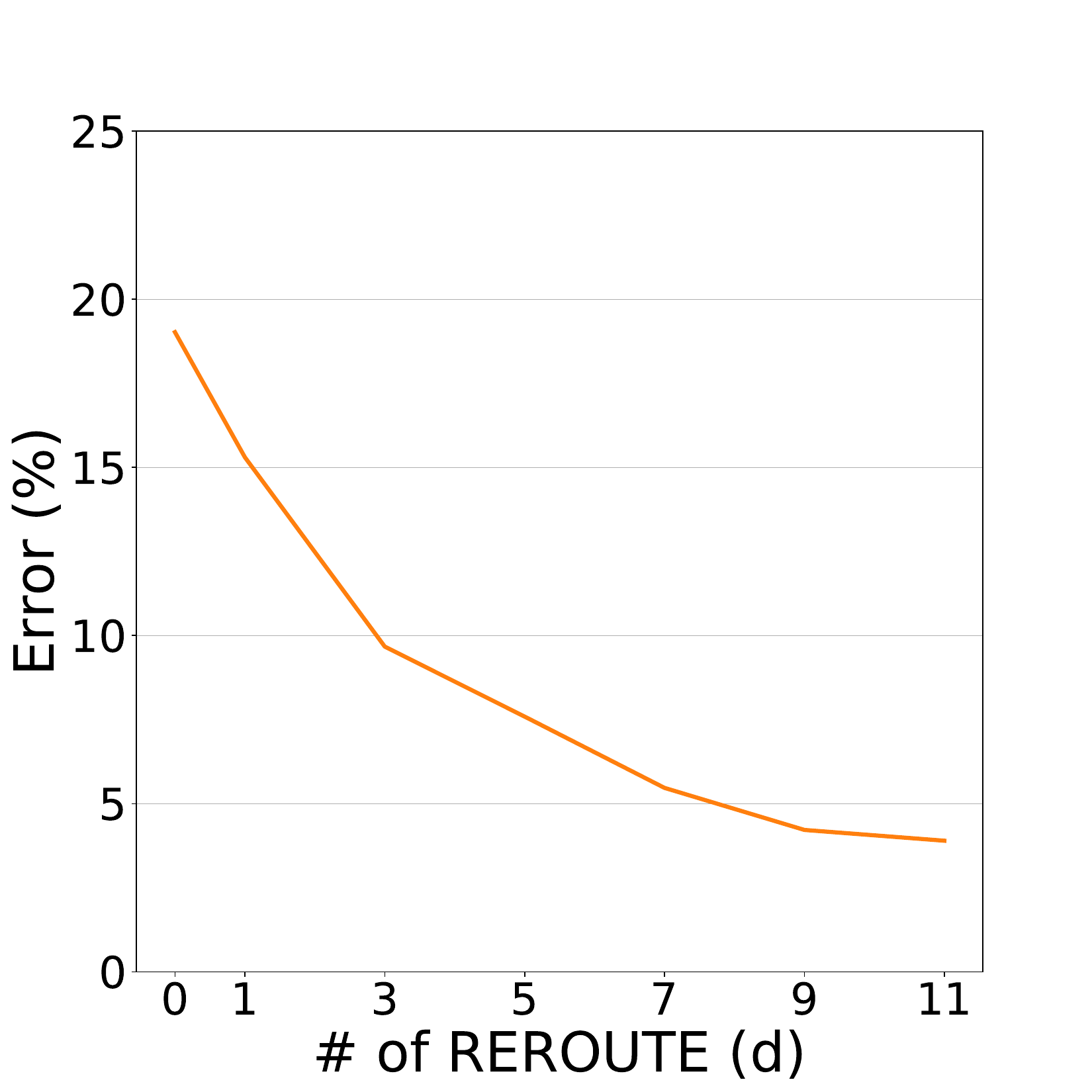}
         \caption{Effect of  $d$}
         \label{fig:effect_of_d}
     \end{subfigure}
      \begin{subfigure}[b]{0.22\textwidth}
     	\includegraphics[width=1.12\textwidth,keepaspectratio]{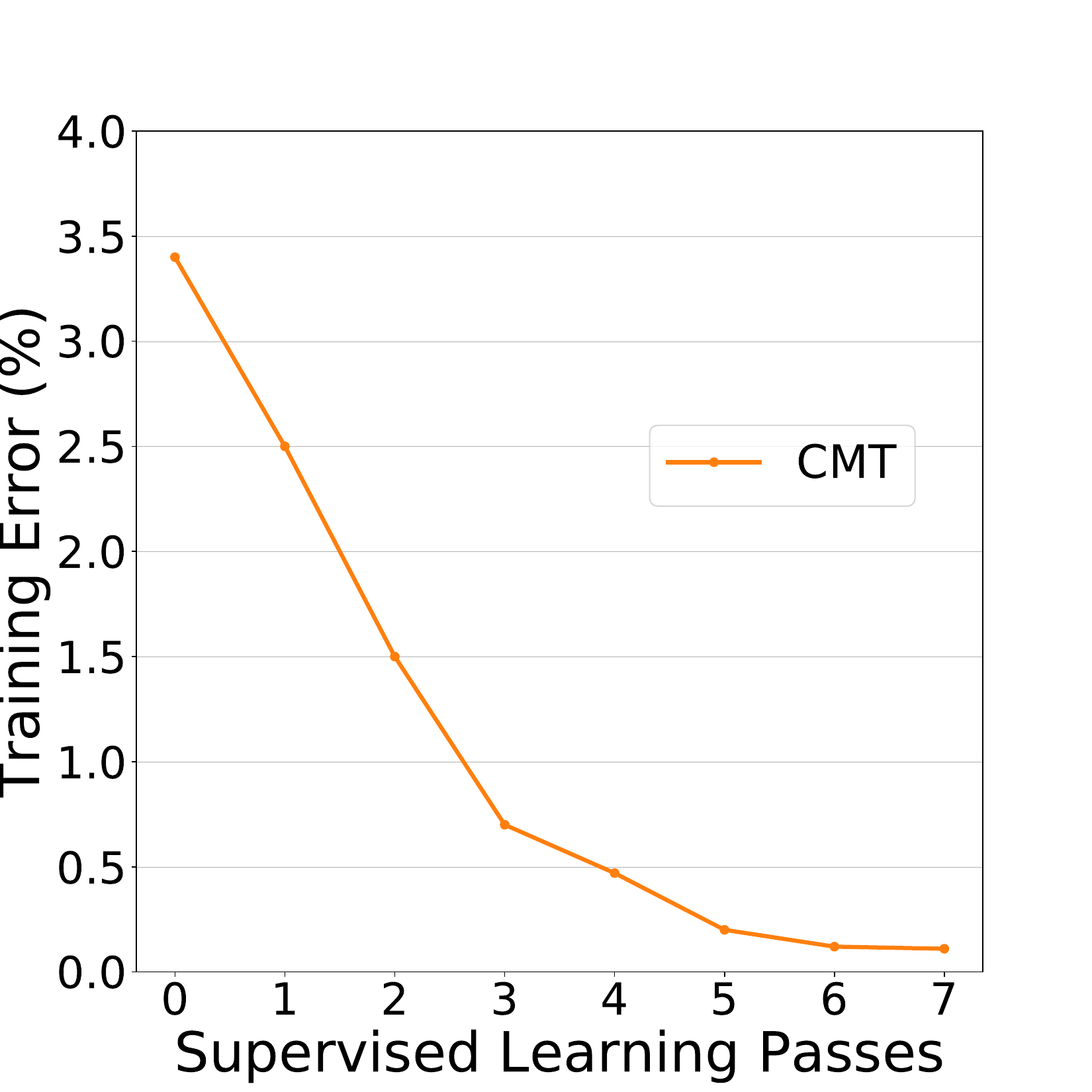}
         \caption{Effect of \# passes}
         \label{fig:wiki_overfit}
      \end{subfigure}
       \begin{subfigure}[b]{0.22\textwidth}
         \includegraphics[width=1.12\textwidth,keepaspectratio]{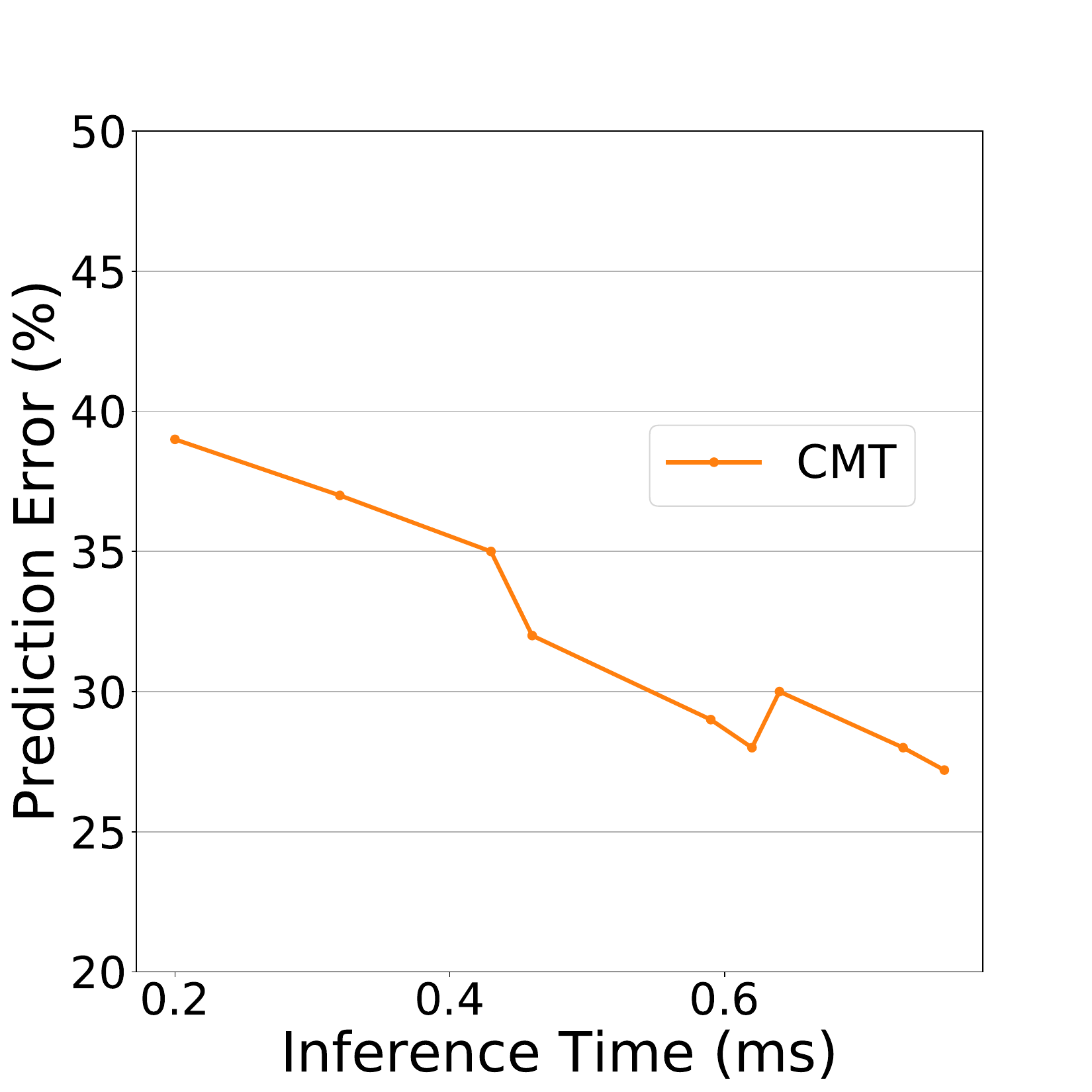}
         \caption{Effect of  $c$}
         \label{fig:effect_of_c}
       \end{subfigure}
       \begin{subfigure}[b]{0.22\textwidth}
         \includegraphics[width=1.12\textwidth,keepaspectratio]{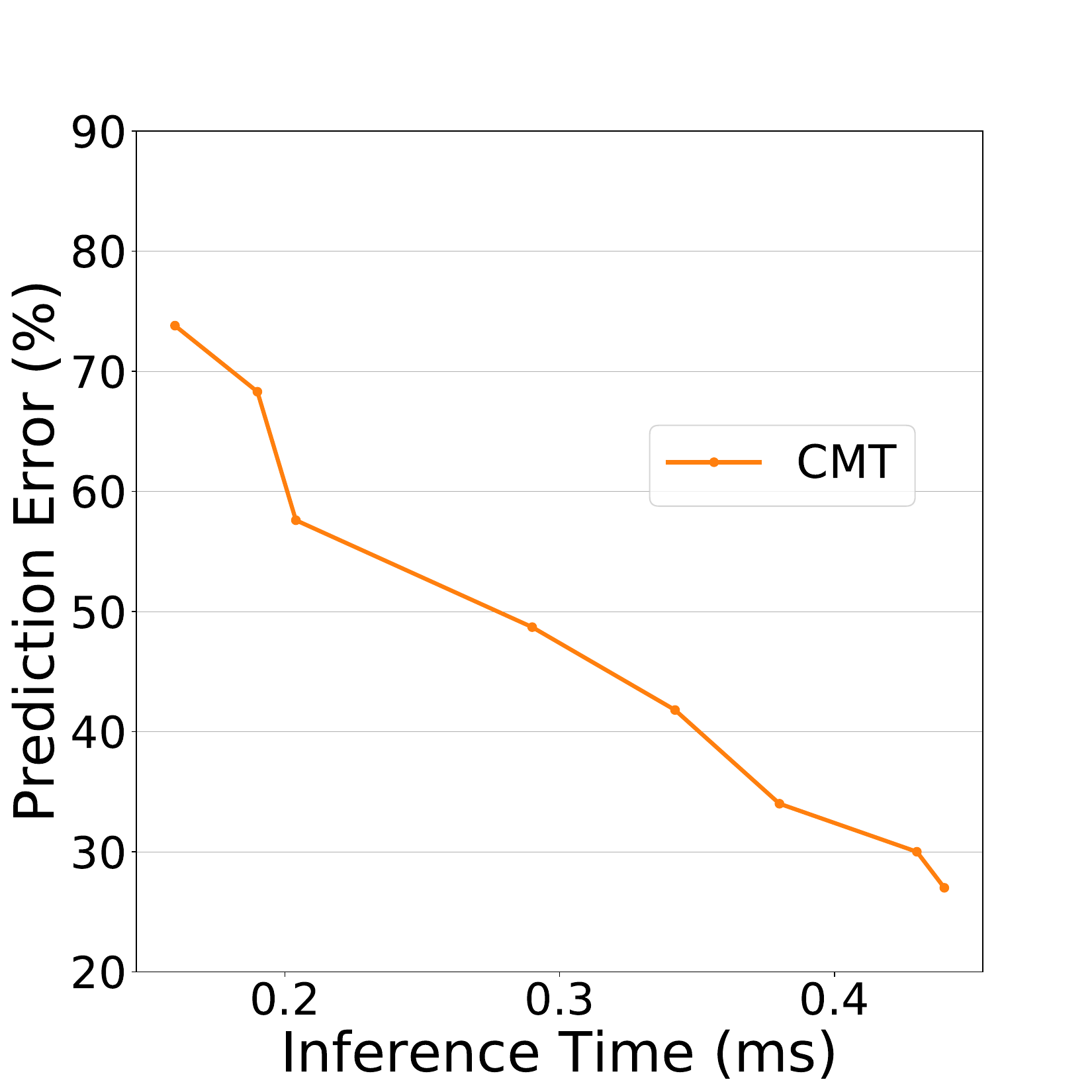}
         \caption{Effect of shots}
         \label{fig:effect_of_shots}
       \end{subfigure}
     \caption{(a) As the number of \textsc{reroute} operations increases \DCMT performs better, asymptoting at ~10 on WikiPara one-shot, (b) the effect of the number of {UPDATE} calls on the training error on WikiPara one-shot, (c) the inference time and prediction error with respect to $c$ on ALOI, and (d) the inference time and prediction error with respect to number of shots on ALOI. }
 \label{fig:ablation_test_plots}
 \vspace{-10pt}
 \end{figure}

We test the self-consistency property
on WikiPara with 
only one training example per class (see \autoref{fig:effect_of_d}).  
We ran \DCMT in an unsupervised
fashion, by only calling \textsc{Insert} and using $-\|x-z.x\|$ as $f(x,z)$ 
to select memories at leafs. We report the self-consistency error
with respect to the number of \textsc{reroute} calls per insertion
(parameter $d$) after four passes over the dataset (tuned using a holdout set).
As $d$ increases, the self-consistency error rapidly drops.

To show that \textsc{Update} is beneficial, we use multiple passes to drive the training error down to nearly zero.
\autoref{fig:wiki_overfit} shows the training error versus the
number of passes on the WikiPara one-shot dataset (on
the $x$-axis, we plot the number of additional passes over the dataset, with zero corresponding to a single pass).
Note that the training error is essentially equal
to the self-consistency error in WikiPara One-shot, hence 
\textsc{Update} further enhances self-consistency due to the extra
\textsc{reroute} operations in \textsc{Update}.

To test the effect of the multiplier $c$ (the leaf memories
multiplier), we switch to the ALOI dataset
\cite{geusebroek2005amsterdam}, which has 100 examples per class enabling
good generalization.  \autoref{fig:effect_of_c} shows that statistical
performance improves with inference time and the value of $c$.  In
Appendix~\autoref{sec:missing_plots}, we include plots showing
statistical and inference time performance vs $c$ in
\autoref{fig:aloi_few_shots_effect_c} with inference time scaling linearly
in $c$ as expected.

Last, we test \DCMT on a series of progressively more difficult
datasets generated from ALOI via randomly sampling $S$ training
examples per label, for $S$ in 1 to 100. ALOI has 1000 unique labels
so the number of memories \DCMT stores scales as $S\times1000$, for
$S$-shot ALOI.  We fix $c = 4$.  \autoref{fig:effect_of_shots} shows
the statistical performance vs inference time as $S$ varies. The
prediction error drops quickly as we increase $S$.
Appendix~\autoref{sec:missing_plots} includes detailed plots.
Inference time increases logarithmically with $S$
(\autoref{fig:aloi_few_shot_time}), matching \DCMT's logarithmic time
operation theory.

\section{Conclusion}
\DCMT provides a new tool for learning algorithm designers by enabling
learning algorithms to work with an unsupervised or reinforced
logarithmic time memory store.  Empirically, we find that \DCMT
provides remarkable unsupervised performance, sometimes beating
previous supervised algorithms while reinforcement provides steady
improvements.


\small

\bibliography{example_paper}

\begin{thebibliography}{42}
\providecommand{\natexlab}[1]{#1}
\providecommand{\url}[1]{\texttt{#1}}
\expandafter\ifx\csname urlstyle\endcsname\relax
  \providecommand{\doi}[1]{doi: #1}\else
  \providecommand{\doi}{doi: \begingroup \urlstyle{rm}\Url}\fi

\bibitem[Andrychowicz \& Kurach(2016)Andrychowicz and
  Kurach]{DBLP:journals/corr/AndrychowiczK16}
Andrychowicz, M. and Kurach, K.
\newblock Learning efficient algorithms with hierarchical attentive memory.
\newblock \emph{CoRR}, abs/1602.03218, 2016.
\newblock URL \url{http://arxiv.org/abs/1602.03218}.

\bibitem[Bartl \& Spanakis(2017)Bartl and Spanakis]{bartl17searchdialogue}
Bartl, A. and Spanakis, G.
\newblock A retrieval-based dialogue system utilizing utterance and context
  embeddings.
\newblock 2017.
\newblock URL \url{http://arxiv.org/abs/1710.05780}.

\bibitem[Beygelzimer et~al.(2006)Beygelzimer, Kakade, and Langford]{CoverTree}
Beygelzimer, A., Kakade, S., and Langford, J.
\newblock Cover trees for nearest neighbor.
\newblock In \emph{Machine Learning, Proceedings of the Twenty-Third
  International Conference {(ICML} 2006), Pittsburgh, Pennsylvania, USA, June
  25-29, 2006}, pp.\  97--104, 2006.
\newblock \doi{10.1145/1143844.1143857}.
\newblock URL \url{http://doi.acm.org/10.1145/1143844.1143857}.

\bibitem[Beygelzimer et~al.(2009)Beygelzimer, Langford, Lifshits, Sorkin, and
  Strehl]{beygelzimer2009conditional}
Beygelzimer, A., Langford, J., Lifshits, Y., Sorkin, G., and Strehl, A.
\newblock Conditional probability tree estimation analysis and algorithms.
\newblock In \emph{Proceedings of the Twenty-Fifth Conference on Uncertainty in
  Artificial Intelligence}, pp.\  51--58. AUAI Press, 2009.

\bibitem[Bhatia et~al.(2015)Bhatia, Jain, Kar, Varma, and
  Jain]{bhatia2015sparse}
Bhatia, K., Jain, H., Kar, P., Varma, M., and Jain, P.
\newblock Sparse local embeddings for extreme multi-label classification.
\newblock In \emph{Advances in Neural Information Processing Systems}, pp.\
  730--738, 2015.

\bibitem[Blum et~al.(1999)Blum, Kalai, and Langford]{blum1999beating}
Blum, A., Kalai, A., and Langford, J.
\newblock Beating the hold-out: Bounds for k-fold and progressive
  cross-validation.
\newblock In \emph{Proceedings of the twelfth annual conference on
  Computational learning theory}, pp.\  203--208. ACM, 1999.

\bibitem[Broder et~al.(2003)Broder, Carmel, Herscovici, Soffer, and Zien]{WAND}
Broder, A.~Z., Carmel, D., Herscovici, M., Soffer, A., and Zien, J.~Y.
\newblock Efficient query evaluation using a two-level retrieval process.
\newblock In \emph{Proceedings of the 2003 {ACM} {CIKM} International
  Conference on Information and Knowledge Management, New Orleans, Louisiana,
  USA, November 2-8, 2003}, pp.\  426--434, 2003.
\newblock \doi{10.1145/956863.956944}.
\newblock URL \url{http://doi.acm.org/10.1145/956863.956944}.

\bibitem[Brouwer(1911)]{Brouwer}
Brouwer, L. E.~J.
\newblock Über abbildungen von mannigfaltigkeiten.
\newblock \emph{Mathematische Annalen}, 71:\penalty0 97--115, 1911.

\bibitem[Cesa{-}Bianchi \& Lugosi(2006)Cesa{-}Bianchi and Lugosi]{noregret}
Cesa{-}Bianchi, N. and Lugosi, G.
\newblock \emph{Prediction, learning, and games}.
\newblock Cambridge University Press, 2006.
\newblock ISBN 978-0-521-84108-5.

\bibitem[Chandar et~al.(2016)Chandar, Ahn, Larochelle, Vincent, Tesauro, and
  Bengio]{chandar2016hierarchical}
Chandar, S., Ahn, S., Larochelle, H., Vincent, P., Tesauro, G., and Bengio, Y.
\newblock Hierarchical memory networks.
\newblock \emph{arXiv preprint arXiv:1605.07427}, 2016.

\bibitem[Choromanska \& Langford(2015)Choromanska and
  Langford]{choromanska2015logarithmic}
Choromanska, A.~E. and Langford, J.
\newblock Logarithmic time online multiclass prediction.
\newblock In \emph{Advances in Neural Information Processing Systems}, pp.\
  55--63, 2015.

\bibitem[Dasgupta \& Sinha(2015)Dasgupta and Sinha]{RPTrees}
Dasgupta, S. and Sinha, K.
\newblock Randomized partition trees for nearest neighbor search.
\newblock \emph{Algorithmica}, 72\penalty0 (1):\penalty0 237--263, 2015.
\newblock \doi{10.1007/s00453-014-9885-5}.
\newblock URL \url{https://doi.org/10.1007/s00453-014-9885-5}.

\bibitem[Datar et~al.(2004)Datar, Immorlica, Indyk, and Mirrokni]{LSH}
Datar, M., Immorlica, N., Indyk, P., and Mirrokni, V.~S.
\newblock Locality-sensitive hashing scheme based on p-stable distributions.
\newblock In \emph{Proceedings of the 20th {ACM} Symposium on Computational
  Geometry, Brooklyn, New York, USA, June 8-11, 2004}, pp.\  253--262, 2004.
\newblock \doi{10.1145/997817.997857}.
\newblock URL \url{http://doi.acm.org/10.1145/997817.997857}.

\bibitem[Daum\'e et~al.(2017)Daum\'e, Karampatziakis, Langford, and
  Mineiro]{daume2016logarithmic}
Daum\'e, III, H., Karampatziakis, N., Langford, J., and Mineiro, P.
\newblock Logarithmic time one-against-some.
\newblock \emph{ICML}, 2017.

\bibitem[Duchi et~al.(2011)Duchi, Hazan, and Singer]{duchi2011adaptive}
Duchi, J., Hazan, E., and Singer, Y.
\newblock Adaptive subgradient methods for online learning and stochastic
  optimization.
\newblock \emph{Journal of Machine Learning Research}, 12\penalty0
  (Jul):\penalty0 2121--2159, 2011.

\bibitem[Freund \& Schapire(1997)Freund and Schapire]{Hedge}
Freund, Y. and Schapire, R.~E.
\newblock A decision-theoretic generalization of on-line learning and an
  application to boosting.
\newblock \emph{J. Comput. Syst. Sci.}, 55\penalty0 (1):\penalty0 119--139,
  1997.
\newblock \doi{10.1006/jcss.1997.1504}.
\newblock URL \url{https://doi.org/10.1006/jcss.1997.1504}.

\bibitem[Geusebroek et~al.(2005)Geusebroek, Burghouts, and
  Smeulders]{geusebroek2005amsterdam}
Geusebroek, J.-M., Burghouts, G.~J., and Smeulders, A.~W.
\newblock The amsterdam library of object images.
\newblock \emph{International Journal of Computer Vision}, 61\penalty0
  (1):\penalty0 103--112, 2005.

\bibitem[Graves et~al.(2016)Graves, Wayne, Reynolds, Harley, Danihelka,
  Grabska{-}Barwinska, Colmenarejo, Grefenstette, Ramalho, Agapiou, Badia,
  Hermann, Zwols, Ostrovski, Cain, King, Summerfield, Blunsom, Kavukcuoglu, and
  Hassabis]{DNC}
Graves, A., Wayne, G., Reynolds, M., Harley, T., Danihelka, I.,
  Grabska{-}Barwinska, A., Colmenarejo, S.~G., Grefenstette, E., Ramalho, T.,
  Agapiou, J., Badia, A.~P., Hermann, K.~M., Zwols, Y., Ostrovski, G., Cain,
  A., King, H., Summerfield, C., Blunsom, P., Kavukcuoglu, K., and Hassabis, D.
\newblock Hybrid computing using a neural network with dynamic external memory.
\newblock \emph{Nature}, 538\penalty0 (7626):\penalty0 471--476, 2016.
\newblock \doi{10.1038/nature20101}.
\newblock URL \url{https://doi.org/10.1038/nature20101}.

\bibitem[Gu et~al.(2018)Gu, Wang, Cho, and Li]{gu18searchnmt}
Gu, J., Wang, Y., Cho, K., and Li, V. O.~K.
\newblock Search engine guided non-parametric neural machine translation.
\newblock In \emph{AAAI}, 2018.

\bibitem[Hodosh et~al.(2013)Hodosh, Young, and Hockenmaier]{hodosh2013framing}
Hodosh, M., Young, P., and Hockenmaier, J.
\newblock Framing image description as a ranking task: Data, models and
  evaluation metrics.
\newblock \emph{Journal of Artificial Intelligence Research}, 47:\penalty0
  853--899, 2013.

\bibitem[Kaiser et~al.(2017)Kaiser, Nachum, Roy, and Bengio]{KaiserNRB17}
Kaiser, L., Nachum, O., Roy, A., and Bengio, S.
\newblock Learning to remember rare events.
\newblock \emph{ICLR}, 2017.

\bibitem[Karnin et~al.(2012)Karnin, Liberty, Lovett, Schwartz, and
  Weinstein]{hyperplane}
Karnin, Z.~S., Liberty, E., Lovett, S., Schwartz, R., and Weinstein, O.
\newblock Unsupervised svms: On the complexity of the furthest hyperplane
  problem.
\newblock In \emph{{COLT} 2012 - The 25th Annual Conference on Learning Theory,
  June 25-27, 2012, Edinburgh, Scotland}, pp.\  2.1--2.17, 2012.
\newblock URL
  \url{http://jmlr.org/proceedings/papers/v23/karnin12/karnin12.pdf}.

\bibitem[Knuth(1997)]{Knuth}
Knuth, D.~E.
\newblock \emph{The art of computer programming, Volume {I:} Fundamental
  Algorithms, 3rd Edition}.
\newblock Addison-Wesley, 1997.
\newblock ISBN 0201896834.
\newblock URL \url{http://www.worldcat.org/oclc/312910844}.

\bibitem[Langford et~al.(2007)Langford, Li, and Strehl]{langford2007vowpal}
Langford, J., Li, L., and Strehl, A.
\newblock Vowpal wabbit online learning project, 2007.

\bibitem[Lin et~al.(2014)Lin, Maire, Belongie, Hays, Perona, Ramanan,
  Doll{\'a}r, and Zitnick]{lin2014microsoft}
Lin, T.-Y., Maire, M., Belongie, S., Hays, J., Perona, P., Ramanan, D.,
  Doll{\'a}r, P., and Zitnick, C.~L.
\newblock Microsoft coco: Common objects in context.
\newblock In \emph{European conference on computer vision}, pp.\  740--755.
  Springer, 2014.

\bibitem[McAuley \& Leskovec(2013)McAuley and Leskovec]{mcauley2013hidden}
McAuley, J. and Leskovec, J.
\newblock Hidden factors and hidden topics: understanding rating dimensions
  with review text.
\newblock In \emph{Proceedings of the 7th ACM conference on Recommender
  systems}, pp.\  165--172. ACM, 2013.

\bibitem[Oquab et~al.(2014)Oquab, Bottou, Laptev, and Sivic]{oquab2014learning}
Oquab, M., Bottou, L., Laptev, I., and Sivic, J.
\newblock Learning and transferring mid-level image representations using
  convolutional neural networks.
\newblock In \emph{Computer Vision and Pattern Recognition (CVPR), 2014 IEEE
  Conference on}, pp.\  1717--1724. IEEE, 2014.

\bibitem[Pedregosa et~al.(2011)Pedregosa, Varoquaux, Gramfort, Michel, Thirion,
  Grisel, Blondel, Prettenhofer, Weiss, Dubourg, et~al.]{pedregosa2011scikit}
Pedregosa, F., Varoquaux, G., Gramfort, A., Michel, V., Thirion, B., Grisel,
  O., Blondel, M., Prettenhofer, P., Weiss, R., Dubourg, V., et~al.
\newblock Scikit-learn: Machine learning in python.
\newblock \emph{Journal of machine learning research}, 12\penalty0
  (Oct):\penalty0 2825--2830, 2011.

\bibitem[Prabhu \& Varma(2014)Prabhu and Varma]{prabhu2014fastxml}
Prabhu, Y. and Varma, M.
\newblock Fastxml: A fast, accurate and stable tree-classifier for extreme
  multi-label learning.
\newblock In \emph{Proceedings of the 20th ACM SIGKDD international conference
  on Knowledge discovery and data mining}, pp.\  263--272. ACM, 2014.

\bibitem[Rae et~al.(2016)Rae, Hunt, Danihelka, Harley, Senior, Wayne, Graves,
  and Lillicrap]{rae2016scaling}
Rae, J., Hunt, J.~J., Danihelka, I., Harley, T., Senior, A.~W., Wayne, G.,
  Graves, A., and Lillicrap, T.
\newblock Scaling memory-augmented neural networks with sparse reads and
  writes.
\newblock In \emph{NIPS}, 2016.

\bibitem[Rashtchian et~al.(2010)Rashtchian, Young, Hodosh, and
  Hockenmaier]{rashtchian2010collecting}
Rashtchian, C., Young, P., Hodosh, M., and Hockenmaier, J.
\newblock Collecting image annotations using amazon's mechanical turk.
\newblock In \emph{Proceedings of the NAACL HLT 2010 Workshop on Creating
  Speech and Language Data with Amazon's Mechanical Turk}, pp.\  139--147.
  Association for Computational Linguistics, 2010.

\bibitem[Rastegari et~al.(2012)Rastegari, Farhadi, and Forsyth]{PDBC}
Rastegari, M., Farhadi, A., and Forsyth, D.~A.
\newblock Attribute discovery via predictable discriminative binary codes.
\newblock In \emph{Computer Vision - {ECCV} 2012 - 12th European Conference on
  Computer Vision, Florence, Italy, October 7-13, 2012, Proceedings, Part
  {VI}}, pp.\  876--889, 2012.
\newblock \doi{10.1007/978-3-642-33783-3_63}.
\newblock URL \url{https://doi.org/10.1007/978-3-642-33783-3_63}.

\bibitem[Salakhutdinov \& Hinton(2009)Salakhutdinov and Hinton]{SemanticHash}
Salakhutdinov, R. and Hinton, G.~E.
\newblock Semantic hashing.
\newblock \emph{Int. J. Approx. Reasoning}, 50\penalty0 (7):\penalty0 969--978,
  2009.
\newblock \doi{10.1016/j.ijar.2008.11.006}.
\newblock URL \url{https://doi.org/10.1016/j.ijar.2008.11.006}.

\bibitem[Santoro et~al.(2016)Santoro, Bartunov, Botvinick, Wierstra, and
  Lillicrap]{santoro2016one}
Santoro, A., Bartunov, S., Botvinick, M., Wierstra, D., and Lillicrap, T.
\newblock One-shot learning with memory-augmented neural networks.
\newblock \emph{arXiv preprint arXiv:1605.06065}, 2016.

\bibitem[Shrivastava \& Li(2015)Shrivastava and Li]{MIPS}
Shrivastava, A. and Li, P.
\newblock Improved asymmetric locality sensitive hashing {(ALSH)} for maximum
  inner product search {(MIPS)}.
\newblock In \emph{Proceedings of the Thirty-First Conference on Uncertainty in
  Artificial Intelligence, {UAI} 2015, July 12-16, 2015, Amsterdam, The
  Netherlands}, pp.\  812--821, 2015.

\bibitem[Simonyan \& Zisserman(2014)Simonyan and Zisserman]{simonyan2014very}
Simonyan, K. and Zisserman, A.
\newblock Very deep convolutional networks for large-scale image recognition.
\newblock \emph{arXiv preprint arXiv:1409.1556}, 2014.

\bibitem[Snell et~al.(2017)Snell, Swersky, and Zemel]{snell2017prototypical}
Snell, J., Swersky, K., and Zemel, R.
\newblock Prototypical networks for few-shot learning.
\newblock In \emph{Advances in Neural Information Processing Systems}, pp.\
  4077--4087, 2017.

\bibitem[Strubell et~al.(2017)Strubell, Verga, Belanger, and
  McCallum]{strubell2017fast}
Strubell, E., Verga, P., Belanger, D., and McCallum, A.
\newblock Fast and accurate entity recognition with iterated dilated
  convolutions.
\newblock \emph{arXiv preprint arXiv:1702.02098}, 2017.

\bibitem[Weinberger et~al.(2005)Weinberger, Blitzer, and Saul]{maxmarginnn}
Weinberger, K.~Q., Blitzer, J., and Saul, L.~K.
\newblock Distance metric learning for large margin nearest neighbor
  classification.
\newblock In \emph{Advances in Neural Information Processing Systems 18 [Neural
  Information Processing Systems, {NIPS} 2005, December 5-8, 2005, Vancouver,
  British Columbia, Canada]}, pp.\  1473--1480, 2005.
\newblock URL
  \url{http://papers.nips.cc/paper/2795-distance-metric-learning-for-large-margin-nearest-neighbor-classification}.

\bibitem[Weston et~al.(2014)Weston, Chopra, and Bordes]{Weston}
Weston, J., Chopra, S., and Bordes, A.
\newblock Memory networks.
\newblock \emph{CoRR}, abs/1410.3916, 2014.
\newblock URL \url{http://arxiv.org/abs/1410.3916}.

\bibitem[Xu \& Schuurmans(2005)Xu and Schuurmans]{Unsuplarge}
Xu, L. and Schuurmans, D.
\newblock Unsupervised and semi-supervised multi-class support vector machines.
\newblock In \emph{Proceedings, The Twentieth National Conference on Artificial
  Intelligence and the Seventeenth Innovative Applications of Artificial
  Intelligence Conference, July 9-13, 2005, Pittsburgh, Pennsylvania, {USA}},
  pp.\  904--910, 2005.
\newblock URL \url{http://www.aaai.org/Library/AAAI/2005/aaai05-143.php}.

\bibitem[Zubiaga(2012)]{zubiaga2012enhancing}
Zubiaga, A.
\newblock Enhancing navigation on wikipedia with social tags.
\newblock \emph{arXiv preprint arXiv:1202.5469}, 2012.

\end{thebibliography}
\bibliographystyle{icml2019}

\newpage
\onecolumn
\appendix

\section{Theorems and proofs}
\subsection{Storage}
\label{sec:storage}
Bounded storage is an easy desiderata to satisfy.  

\begin{claim}For any $T>0$, a contextual memory tree after $T$ insertions requires only $O(T)$ storage.\end{claim}
\begin{proof}
The hashmap is $O(T)$. 
The number of internal nodes is bounded by the number of leaf nodes. Since
every leaf node has at least one unique memory, the storage requirement for internal nodes is $O(T)$, and so is the storage requirement for the leaves.
\end{proof}

\subsection{Incrementality}
\label{sec:incrementality}
By observation, all contextual memory tree algorithms are incremental
so the overall operation is incremental as long as the underlying learning
algorithms for the learning scorer and routers are incremental.  In fact, the
contextual memory tree is online so long as the underlying learning
algorithms are online.

\subsection{Partitioning}
\label{proof:partition}
Here we prove the partition bound (Theorem~\ref{thm:partition}).
\begin{proof}
Let $R_t$ and $L_t$ be the number of memories in the right and left subtree respectively, at the start of round $t$ for which we are proving the theorem. 

Observe that if
\[
\alpha \log \frac {L_t} {R_t} > 1 - \alpha,
\]
or, equivalently, if
\[
\frac {L_t}{R_t} > e^{\frac{1 - \alpha}{\alpha}},
\] or equivalently, if
\begin{align}
\frac{L_t}{N_t} > \frac{1}{1 + \exp(1-\frac{1}{\alpha})},
\end{align} where $N_t = R_t + L_t$ we always have $y=1$.

A symmetric argument shows that if
\begin{align}
\frac{R_t}{N_t} > \frac{1}{1 + \exp(1-\frac{1}{\alpha})},
\end{align} we always have $y=-1$.

Denote $\kappa = \frac{1}{1 + \exp(1-\frac{1}{\alpha})}$.  Note that $\kappa < 1$. We claim that for any $t$, we have
\begin{align}
R_t \leq (1-p_t)\kappa N_t + (1-\kappa) + p_t N_t,  \text{ and } L_t \leq (1-p_t)\kappa N_t + (1-\kappa) + p_t N_t,
\end{align} 
where $p_t$ is the progressive training error at the beginning of round $t$.
We prove the claim by induction on $t$. The base case holds by inspection, assuming $L_2 = R_2 = 1$ and $p_2 = 0$ (i.e., by simply initializing all leaf with a default example).

Assume that the claim holds for step $t$, and consider step $t+1$.

Below we first consider the first case: (1) $R_t > \kappa N_t$. 

Note that in this case, we always have $y = -1$. Whether or not we route the example to the left depends on whether or not the router makes a post-update mistake.  Hence, we discuss two sub-cases below. 

(a) The router does not make a mistake here. In this case, the router routes the example to the left. Since no mistake happens in this round, we have $p_{t+1}N_{t+1} = p_t N_t$, i.e., the total number of mistakes remain the same.  Then, we have:
\begin{align}
R_{t+1} = R_t \leq (1-p_t)\kappa N_t + (1-\kappa) + p_t N_t \leq (1-p_{t+1})\kappa N_{t+1} + (1-\kappa) + p_{t+1}N_{t+1},
\end{align} where the inequality comes from the fact that $N_{t+1} = N_t + 1 > N_t$.

Now we consider the second sub-case here.

(b) The router does make a mistake. In this case, the router routes the example to the right. Note that in this case, we have $p_{t+1}N_{t+1} = p_t N_t + 1$, i.e., the total number of mistakes increases by one.  Hence, we have for $R_{t+1}$:
\begin{align}
R_{t+1} & = R_{t}+1 \leq (1-p_t)\kappa N_t + (1-\kappa) + p_tN_t + 1 \nonumber\\
& = \kappa N_t - \kappa p_t N_t + (1-\kappa) + p_tN_t + 1 \nonumber\\
& =  \kappa N_t - \kappa p_{t+1}N_{t+1} + \kappa + (1-\kappa) + p_tN_t + 1\nonumber\\
& = \kappa N_{t+1} - \kappa p_{t+1} N_{t+1} + (1-\kappa) + p_{t+1}N_{t+1},
\end{align} where the second equality uses the fact that $\kappa p_{t+1}N_{t+1} = \kappa p_tN_t + \kappa$.

With case (a) and case (b), we can conclude that for case (1) where $R_t > \kappa N_t$, we have:
\begin{align}
R_{t+1} \leq (1-p_{t+1})\kappa N_{t+1} + (1-\kappa) + p_{t+1}N_{t+1}.
\end{align}

Now we consider the second case (b): $R_t \leq \kappa N_t$.
In this case, regardless of where the example routes, we always have:
\begin{align}
R_{t+1} \leq R_t + 1 \leq \kappa N_t + 1 = \kappa (N_{t+1} - 1) + 1 = \kappa N_{t+1} + 1 - \kappa.
\end{align}
Note that since $\kappa < 1$, we must have $p_{t+1}N_{t+1} \geq p_{t+1}\kappa N_{t+1}$. Hence we have
\begin{align}
R_{t+1} & \leq \kappa N_{t+1} + 1 - \kappa \nonumber\\
& \leq \kappa N_{t+1} + 1 - \kappa + p_{t+1}N_{t+1} - \kappa p_{t+1}N_{t+1} = (1-p_{t+1})\kappa N_{t+1} + (1-\kappa) + p_{t+1}N_{t+1}.
\end{align}

With case (1) and case (2), we can conclude that for $R_{t+1}$, we always have:
\begin{align}
R_{t+1} \leq  (1-p_{t+1})\kappa N_{t+1} + (1-\kappa) + p_{t+1}N_{t+1}.
\end{align} A symmetric argument implies
\begin{align}
L_{t+1 }\leq  (1-p_{t+1})\kappa N_{t+1} + (1-\kappa) + p_{t+1}N_{t+1}.
\end{align}
By induction, we prove our claim.

Now given $L_{t}\leq  (1-p_{t})\kappa N_{t} + (1-\kappa) + p_{t}N_{t}$, we divide $N_t$ on both sides to get:
\begin{align}
L_{t}/N_t \leq (1- p_t) \kappa + \frac{1-\kappa}{N_t} + p_t.
\end{align} Multiplying both sides by $-1$ and adding $1$, we get:
\begin{align}
1 - L_t/N_t = \frac{R_t}{N_t} &\geq 1 - (1-p_t)\kappa + \frac{\kappa - 1}{N_t} - p_t  \nonumber\\
& =  (1- p_t) - (1-p_t)\kappa + \frac{\kappa - 1}{N_t} \nonumber\\
& =  (1-p_t)(1-\kappa) + \frac{\kappa - 1}{N_t}.
\end{align} As $\kappa > 0$, we get:
\begin{align}
R_t/N_t \geq (1-p_t)(1-\kappa) - \frac{1}{N_t}.
\end{align}
By symmetry, we have:
\begin{align}
L_t/N_t \geq (1-p_t)(1-\kappa) - \frac{1}{N_t}.
\end{align}
Substituting $\kappa$ in, we get:
\begin{align}
\min\{L_t / N_t, R_t/N_t\}  \geq  (1-p_t)\frac{1}{\exp(\frac{1-\alpha}{\alpha}) + 1} - \frac{1}{N_t}.
\end{align}
\end{proof}

\subsection{Depth of $K$-partitions}
\label{proof:depth}
Next we prove a depth bound given $K$-partitions.  
\begin{lemma}
A tree on $T$ points with a $K$-partition at every internal node has depth at most $K \log T$.
  \end{lemma}
\begin{proof}
  By assumption, each internal node routes at least a $1/K$ fraction of
  incident points in either direction, hence at most a $1-1/K$ fraction of
  points are routed the other direction.  As a consequence, at a depth
  $d$ a node has at most $t (1-1/K)^d$ memories beneath it.  The
  deepest internal node in the tree satisfies:
  $$ T (1-1/K)^d \geq 1$$
  rearranging, we get: 
  $$ T \geq  \left(\frac{1}{1-1/K}\right)^d $$
  Taking the log of both sides, we get: 
  $$ \log T \geq d \log \left(\frac{1}{1-1/K}\right) $$
  which implies
  $$d \leq \frac{\log T}{\log \left(\frac{1}{1-1/K}\right)}. $$
  Using $-\log(1-x) \geq x$ for $0\leq x<1$, we get
  $$d \leq K \log T.$$
  \end{proof}

\subsection{Computational bound proof}
\label{proof:computation}
Now we prove Theorem~\ref{thm:computation}.

\begin{proof}
  We assume that $d$ is constant.
  From the depth bound, \algorithmname{Remove} is $O(K\log T)$.
  \algorithmname{InsertLeaf} is $O(1)$ if the guard on line 2 is false.  If the guard is
  true, then we know that $|v.m| > c \log T$ and $|v.m| -1 \leq c \log T$
  since otherwise it would have been triggered on a previous
  insertion.  Hence, line 5 executes $O(c \log T)$ times, with each invocation of \algorithmname{Insert} taking $O(1)$ time in this case 
as the while loop in line 1 is executed only once.

  \algorithmname{Insert}$(\cdot,\cdot,0)$ takes $O((K+c)\log T)$ from the depth bound and the complexity of \algorithmname{InsertLeaf}.
  Thus the computational complexity of \algorithmname{Reroute} is $O((K + c) \log T)$. \algorithmname{Update} takes $O(1)$ time, followed by $d$ invocations of \algorithmname{Reroute}, making it $O((K + c) \log T)$ time as well.
   \algorithmname{Insert}$(\cdot,\cdot,d)$ takes $O((K+c)\log T)$, followed by $d$ invocations of \algorithmname{Reroute}, making its total complexity $O((K+c)\log T)$.

  \algorithmname{Query} calls \algorithmname{Path} at most twice and then pays $O(c \log T)$ computation to find the top $k$ memories for the query.  The complexity of \algorithmname{Path} is $O(K\log T)$, making the overall complexity of \algorithmname{Query}  $O((K+c)\log T)$.  
\end{proof}

\subsection{Self-Consistency}
\label{sec:self}
Let us recall the definition of self-consistency.
\begin{definition} 
A {\normalfont \DCMT} is {\bf self-consistent} if for all $z$ with a unique $z.x$, $z=$\algorithmname{Query}$(z.x,1,0)$.
\end{definition}
It is easy to see that self-consistency holds for any $z$ immediately after insertion.

\begin{lemma}If $z = \arg \max_{z'}f(z.x,z')$, then $z=$\algorithmname{Query}$(z.x,1,0)$ immediately after \algorithmname{Insert}$(\Root, z, 0)$.
\end{lemma}
\begin{proof}
  By construction, the updates in line 4 of \algorithmname{Insert} do
  not affect the routers at nodes closer to the root.
  Therefore, since \algorithmname{Insert} line 6 and
  \algorithmname{Path} line 3 are identical, both
  \algorithmname{Insert} and \algorithmname{Query} walk through the
  same internal nodes.  At \algorithmname{InsertLeaf},
  the last execution of line 5 is for $z$ and hence any newly created
  internal node also routes in the same direction.  Once a leaf is
  reached, $z = \arg \max_{z'}f(z.x,z')$ implies the claim
  follows.
\end{proof}

Achieving self-consistency for all $z$ simultaneously is more
difficult since online updates to routers can invalidate pre-existing
self-consistency.  Nevertheless, the combination of the \algorithmname{Reroute}
operation and the convergence of learning algorithms at internal nodes
leads to asymptotic self-consistency.

\begin{definition}A {\bf convergent learning algorithm} satisfies,
for all input distributions $D$ and all update sequences,
\[
\mathbf{P}_{x\sim D}[ g_{t}(x) \neq g_{t-1}(x)] = 0,
\]
in the limit as $t\rightarrow \infty$.
  \end{definition}
Restated, a convergent learning algorithm is one that disturbs fewer
predictions the more updates that it gets.  This property is an
abstraction of many existing update rules with decaying learning
rates.  
\begin{theorem}\label{thm:self} For all contextual memory trees $T$, if $z = \arg \max_{z'}T.f(z.x,z')$ for all $z$, and all routers $g$ are convergent under the induced sequence of updates, then in the limit as ${T}.d\rightarrow \infty$, $T$ is self-consistent almost surely.
\end{theorem}
\begin{proof}
The proof operates level-wise.  The 
uniform \algorithmname{Reroute} operation and the fact that the learning algorithm at the root is convergent by assumption guarantees that the root eventually routes in a self-consistent fashion almost surely.  Once the root converges, the same logic applies recursively to every internal node, for the distribution of memories induced at the node.
To finish the proof, we just use the assumption that $z = \arg \max_{z'}f(z.x,z')$.
\end{proof}

Asymptotic
self-consistency is a relatively weak property so we also study
self-consistency empirically in section~\ref{sec:ablation}.

\subsection{Learning}
\label{sec:learning}
Finding a good partition from a learning perspective is plausibly more
difficult than finding a good classifier.  For example, in a vector
space finding a partition with a large margin which separates input
points into two sets each within a constant factor of the original in
size is an obvious proxy.  The best results for this
problem~\cite{Unsuplarge,hyperplane} do not scale to large datasets or
function in an online fashion.

For any given node we have a set of incident samples which cause
updates on \algorithmname{Insert} or \algorithmname{Update}.  Focusing on \algorithmname{Update} at a single node, the
natural function to optimize is a form of balanced expected reward.
If $r_a$ and $p_a$ are the rewards and probabilities of taking action
$a$, then a natural objective is:
\begin{align}\label{eq:obj}
\arg \max_g E_{x\sim D}(1-\alpha)r_{g(x)} - \alpha \log p^g_{g(x)}
\end{align}
where $p^g_{a} = \Pr_{x \sim D}(g(x)=a)$ is the probability that $g$
chooses direction $a$ as induced by samples over $x$.  This objective
both maximizes reward and minimizes the frequency of the chosen
action, implying a good solution sends samples in both directions.

The performance of the partitioner is dependent on the classifier $g$
which optimizes importance weighted binary classification.  In
particular, we evaluate the performance of $g$ according to:
\[ \hat{E}_{x,y,i} i I(g(x) \neq y)\]
with the goal of $g$ minimizing the empirical importance weighted loss
over observed samples.

Next we prove a basic sanity check theorem about the asymptotics of
learning a single node.  For this theorem, we rely upon the notion of
a no-regret~\cite{noregret} $g$ which is also convergent.  Common
no-regret algorithms like Hedge~\cite{Hedge} are also convergent for
absolutely continuous $D$ generating events.  The following theorem
relies on the 
\begin{theorem}
  For all absolutely continuous distributions $D$ over updates with $d=0$ reroutes and for all compact convergent no-regret $g$: $$\lim_{t\rightarrow \infty} g_t$$ exists and is a local maxima of \eqref{eq:obj}.
\end{theorem}
The proof is in Appendix~\ref{proof:learning}.  Here, convergent $g$
is as defined in section~\ref{sec:self} and compact $g$ refers to the
standard definition of a compact space for the parameterization of
$g$.  

It's important to note that the $d=0$ requirement is inconsistent with
the $d\rightarrow\infty$ requirement for self-consistency.  This
tradeoff is fundamental: a learning process that is grounded in
unsupervised updates (as for self-consistency) is fundamentally
different from a learning process grounded in rewards (as for the
learning update).  If these two groundings happen to agree then
compatibility exists as every unsupervised update is consistent with a
reward update.

This theorem shows that the optimization process eventually drives to
a local maxima of~\eqref{eq:obj} providing a single node semantics.
Since every node optimizes independently, the joint system therefore
eventually achieves convergence over 1-step routing deviations.

\subsection{Learning proof}
\label{proof:learning}
\begin{proof} 
  Consider without loss of generality the root node of the tree, and
  then apply this argument recursively.

  Since $g$ is no-regret the $g$ minimizing \eqref{eq:obj} for any
  observed $p$ eventually wins.  Since the $D$ producing updates is
  absolutely continuous, convergence of $g$ implies convergence of $p$
  and the $g,p$ system is compact since $g$ is compact and $p$ is
  compact.  Given this, a $g,p$ pair maps to a new $g,p$ pair
  according to the dynamics of the learning algorithm.  

  Brouwer's fixed point theorem~\cite{Brouwer} hence implies that
  there exists a $g,p$ pair which is a fixed point of this process.
  Since $g$ is no-regret, the system must eventually reach such a
  fixed point (there may be many such fixed points in general).

  For a given $g$, let 
  $\Phi^{(g)} = E_{x\sim D}\left[(1-\alpha)r_{g(x)} - \alpha \log p^g_{g(x)}\right]$ 
  be the objective in equation \eqref{eq:obj} and define
  \begin{align*}
    r^{(g)}_\text{Left} & = (1-\alpha)r_\text{Left} - \alpha \log p^g_\text{Left}\\
    r^{(g)}_\text{Right} & = (1-\alpha)r_\text{Right} - \alpha \log p^g_\text{Right}.
  \end{align*}
  Using this definition, we can define:
  \begin{align*}
    y^{(g)} &=  r^{(g)}_\text{Right} - r^{(g)}_\text{Left}\\
    & = (1-\alpha)(r_\text{Right} - r_\text{Left}) - \alpha (\log p^g_\text{Right} + \log p^g_\text{Left})\\
    & = (1-\alpha)(r_\text{Right} - r_\text{Left}) + \alpha \log \frac{p^g_\text{Left}}{p^g_\text{Right}} \\
    & \stackrel{a.s.}{=} (1-\alpha)(r_\text{Right} - r_\text{Left}) + \alpha \lim_{\substack{t \rightarrow \infty \\ S \sim D^t}} \log \frac{\sum_{x \in S} I(g(x)=\text{Left})}{\sum_{x \in S} I(g(x)=\text{Right})}.
  \end{align*}
  Assume wlog that $r^{(g)}_\text{Right} > r^{(g)}_\text{Left}$ such that 
  $|y^{(g)}| = y^{(g)}$.  Examining Line 5 of \algorithmname{Update}, 
  for a fixed $g$ (i.e. g.update() has converged), taking expectations 
  wrt $p$ over $a$, and denoting $H$ as the complete empirical history 
  of the node,
  \begin{align*}
    E_p y&  = (1-\alpha)E_{a \sim \vec{p}} \left(\frac{r_\text{Right}I(a=\text{Right})}{p(\text{Right})} - \frac{r_\text{Left}I(a=\text{Left})}{p(\text{Left})}\right) + E_{x \sim D}\log \frac{\sum_{x \in H} I(g(x)=\text{Left})}{\sum_{x \in H} I(g(x)=\text{Right})}\\
    & \stackrel{t \to \infty}{=} (1-\alpha)(r_\text{Right} - r_\text{Left}) + E_{\substack{x \sim D \\ S \sim D^t}}\log \frac{\sum_{x \in S} I(g(x)=\text{Left})}{\sum_{x \in S} I(g(x)=\text{Right})}.
  \end{align*}
  In other words, $\lim_{t \rightarrow \infty} E_p y \stackrel{a.s.}{=} y^{(g)}$.  
  The expected loss of $g$ then converges to:
  \begin{align*}
    & E \left[ |y^{(g)}| I(g(x) \neq \sign(y^{(g)})) \right] \stackrel{a.s.}{=} \Phi^{(g)}
  \end{align*}
  proving the theorem.
\end{proof}

\section{Experimental Details}
\subsection{Datasets}
\label{sec:datasets}
\begin{table}[h!]
\begin{center}
\resizebox{0.7\textwidth}{!}{ 
 \begin{tabular}{c c c c } 
 \toprule
 dataset & task & classes & examples \\ 
 \midrule
 ALOI & Visual Object Recognition  & 10$^3$ & 10$^5$ \\ 
 WikiPara ($S$-shot) &  Language Modeling  & 10$^4$ &  $S\times$10$^4$  \\
 ImageNet ($S$-shot) & Visual Object Recognition & 2$\times$ 10$^4$ & 2$S\times$10$^4$  \\
 Pascal & Image-Caption Q\&A & / & 10$^3$ \\
 Flickr-8k & Image-Caption Q\&A & / &  8$\times$10$^3$ \\ 
 MS COCO & Image-Caption Q\&A & /  & 8$\times$10$^4$ \\ 
 \bottomrule
 \end{tabular}}
 \end{center}
 \caption{ Datasets used for experimentation on multi-class and Image Retrieval}
\label{tab:dataset_multi_class}
\end{table}

\begin{table}[h!]
\begin{center}
\resizebox{1.\textwidth}{!}{ 
    \rowcolors{2}{gray!25}{white}
 \begin{tabular}{c c c c c c c} 
 \toprule
 dataset & \# Training & \# test & \# Categories & \# Features & Avg \# Points/Label & Avg \# Labels/Point \\ 
 \midrule
 RCV1-2K & 623847  & 155962  &  2456 & 47236 & 1218.56 & 4.79  \\ 
 AmazonCat-13K  & 1186239 & 306782 & 13330 & 203882 & 448.57 & 5.04   \\
 Wiki10-31K & 14146 & 6616 & 30938 & 101938 & 8.52 & 18.64   \\
 \bottomrule
 \end{tabular}}
 \end{center}
 \caption{ Extreme Multi-Label Classification datasets used for experimentation}
\label{tab:dataset_multi_label}
\end{table}

\autoref{tab:dataset_multi_class} summarizes the datasets used in Multi-class classification and image retrieval experimentations. 
ALOI \cite{geusebroek2005amsterdam} is a color image collection of one-thousand small objects. We use the same train and test split and feature representation as Recall Tree \cite{daume2016logarithmic}. The few-shot ImageNet datasets are constructed from the whole ImageNet that has 20,000 classes and $10^7$ training examples. We use the same train and test split as Recall Tree \cite{daume2016logarithmic}. The features of images are extracted from intermediate layers of a convolutional neural network trained on the ILVSRC2012 \cite{oquab2014learning}. To construct a $S$-shot ImageNet dataset, we randomly sample $S$ training examples for each class. A $S$-shot ImageNet dataset hence has a $20000\times S$ many training examples. 

Pascal sentence dataset consists of 1000 pairs of image $I\in\mathbb{R}^{300\times 180}$ and the corresponding description of the image. We compute HoG feature $y$ for each image $I$ and token occurrences $y\in\mathbb{R}^{2^{20}}$ for each description using Scikit-learn's Hashing functionality. The resulting feature $x$ is high dimensional but extremely sparse. We randomly split the dataset into a training set consisting of 900 pairs of images and their descriptions and a test set with the remaining data. A memory $z=(x,y)$ here consists of the image feature $y$ and the descriptions' feature $x$.  During inference time, given a query $x$ (i.e., a description of an unknown image), \DCMT{} retrieves a memory $z'=(x',y')$, such that the image $y'$ associated with the memory $z'$ is as similar to the unknown image associated with the test query $x$.  Given two memories $z$ and $z'$, the reward signal is defined as $r(z_y, z'_y ) = \langle z_y, z'_y \rangle $. The Flickr8k dataset consists of 8k images and 5 sentences descriptions for each image. Similar to Pascal, we compute HoG feature $y$ for each image and hashing feature $x$  for its 5-sentence description. The MS COCO image caption dataset consists of 80K images in training set, 4000 images in validation set and testing set. We extract image feature $y$ from a fully connected layer in a VGG-19 \cite{simonyan2014very} pre-trained on ILSVRC2012 dataset . We use hashing feature 
$x$ for image captions. 

\autoref{tab:dataset_multi_label} summarizes the datasets used for multi-label classification task. All three datasets are obtained from the Extreme Classification Repository (\url{http://manikvarma.org/downloads/XC/XMLRepository.html}). 

All datasets that we used throughout this work are available at (url will be provided here).

\begin{table*}[t]
\begin{center}
\resizebox{0.8\textwidth}{!}{ 
  \centering
    \rowcolors{2}{gray!25}{white}
    \begin{tabular}{llllllllllr}
       \toprule
      &  \# unsupervised passes  &  \# supervised passes &  c  & d & $\alpha$  \\ 
     \midrule
    {ALOI} & 1  &  2  & 4 & 5 & 0.1   \\
   Few-shot WikiPara &  1  &   1  & 4  & 5  & 0.9   \\
   {Few-shot ImageNet} & 1 &  1 &  4 & 3 & 0.9  \\
   RCV1-1K & 1 & 3 & 2 & 3 & 0.9 \\
   AmazonCat-13K & 1 & 3 & 2 & 3 & 0.9 \\
   Wiki10-31K & 1 & 3 & 2 & 3 & 0.9  \\
   Pascal & 1 & 1 & 10 & 1 & 0.9  \\
   Flickr & 1 & 1 & 10 & 1 & 0.9 \\
   MS COCO & 1 & 1 & 10 & 1 & 0.9  \\
    \bottomrule
    \end{tabular}}
\end{center}
\vspace{-5pt}
\caption{Key parameters used for \DCMT{} for our experiments}
\label{tab:parameters}
\end{table*}

\subsection{Extra Plots in Sec.~\ref{sec:ablation}}
\label{sec:missing_plots}
\begin{figure}[t!]
	\centering
	\begin{subfigure}[l]{0.35\textwidth}
        \includegraphics[width=1.\textwidth,keepaspectratio]{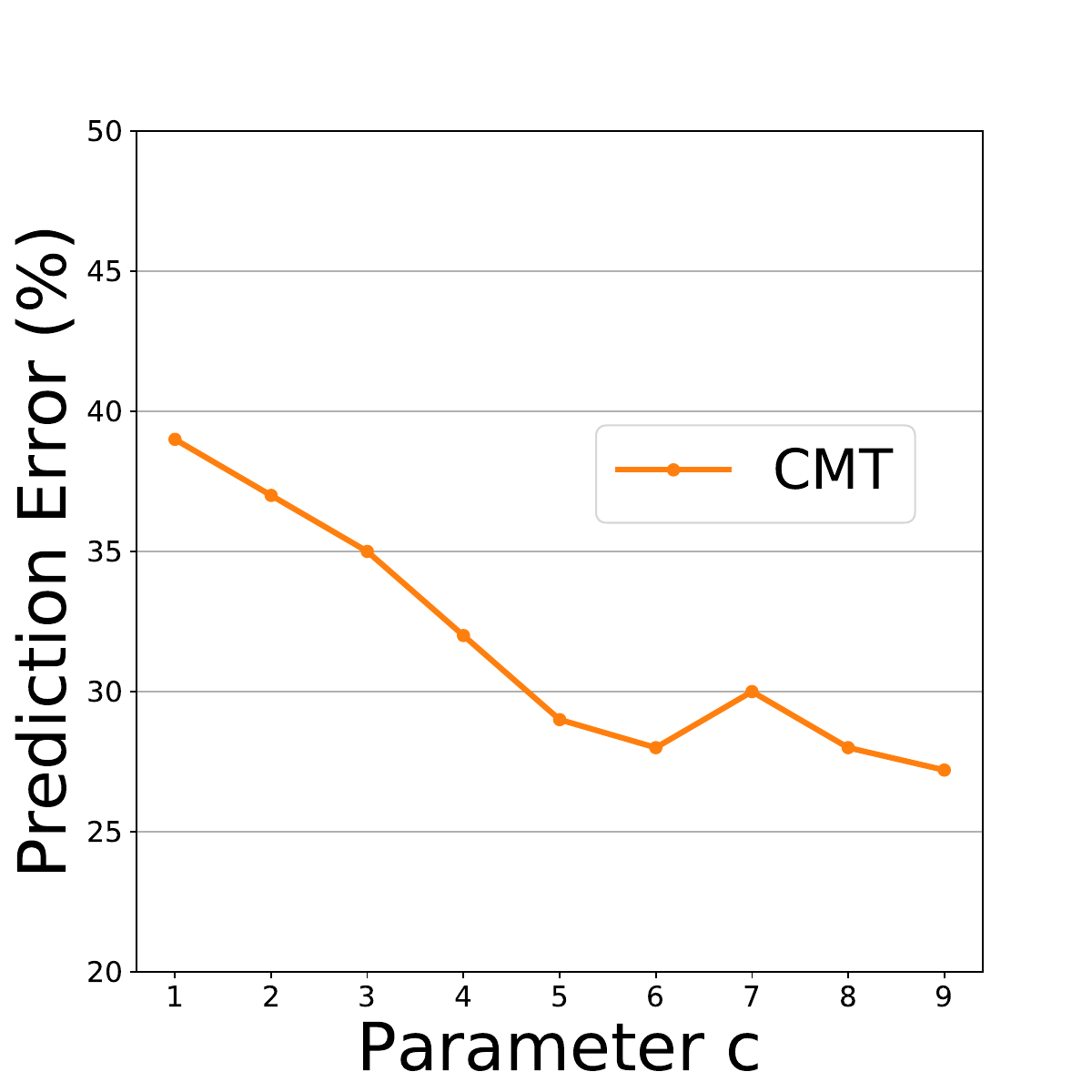}
        \caption{Prediction Error}
        \label{fig:aloi_few_shot_c_perf}
    \end{subfigure}
    \begin{subfigure}[l]{0.35\textwidth}
        \includegraphics[width=1.\textwidth,keepaspectratio]{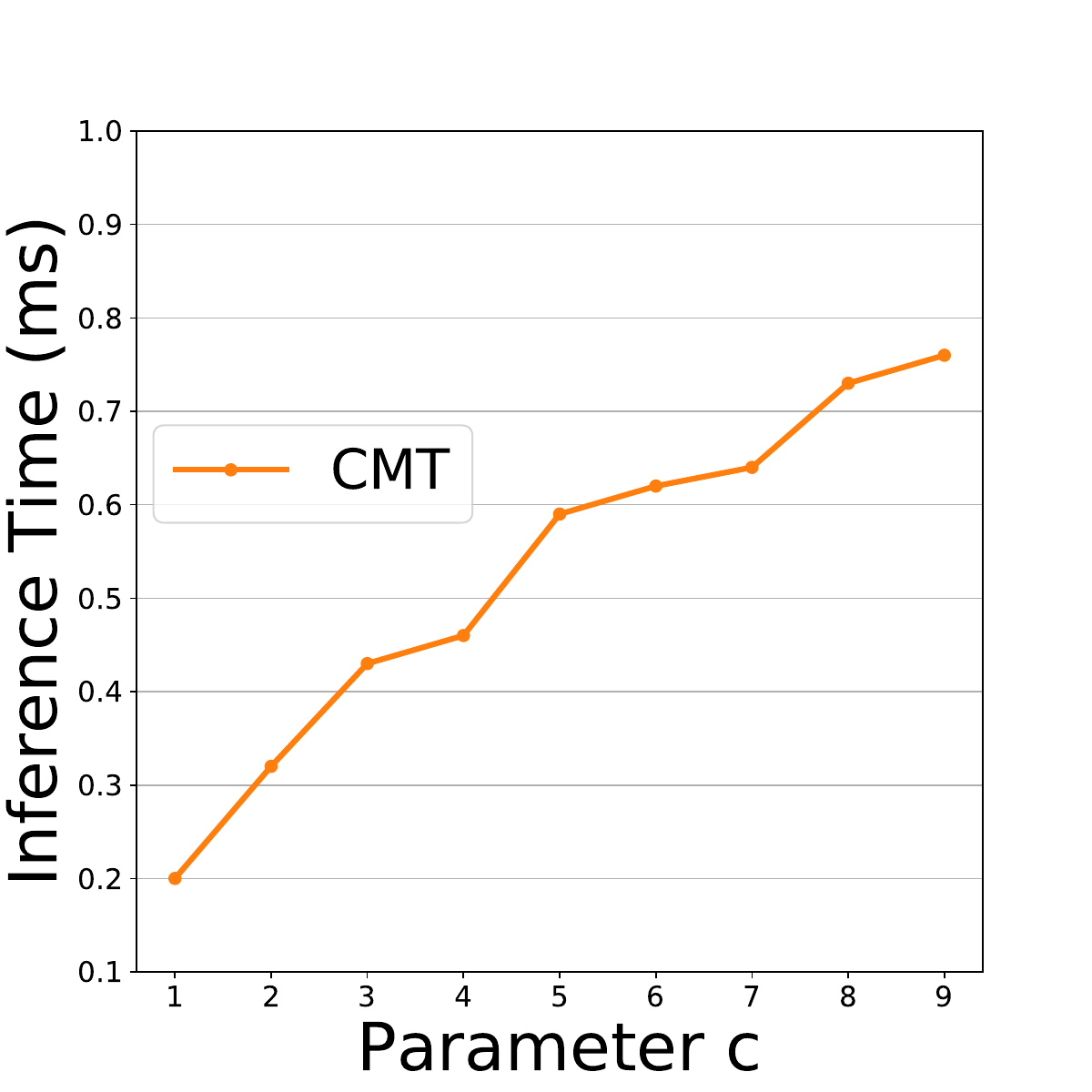}
        \caption{Inference Time}
        \label{fig:aloi_few_shot_c_time}
    \end{subfigure}
    \caption{Performance and inference time of CMT versus the number of examples per leaf}
\label{fig:aloi_few_shots_effect_c}
\end{figure}

\begin{figure}[t!]
	\centering
	\begin{subfigure}[l]{0.35\textwidth}
        \includegraphics[width=1\textwidth,keepaspectratio]{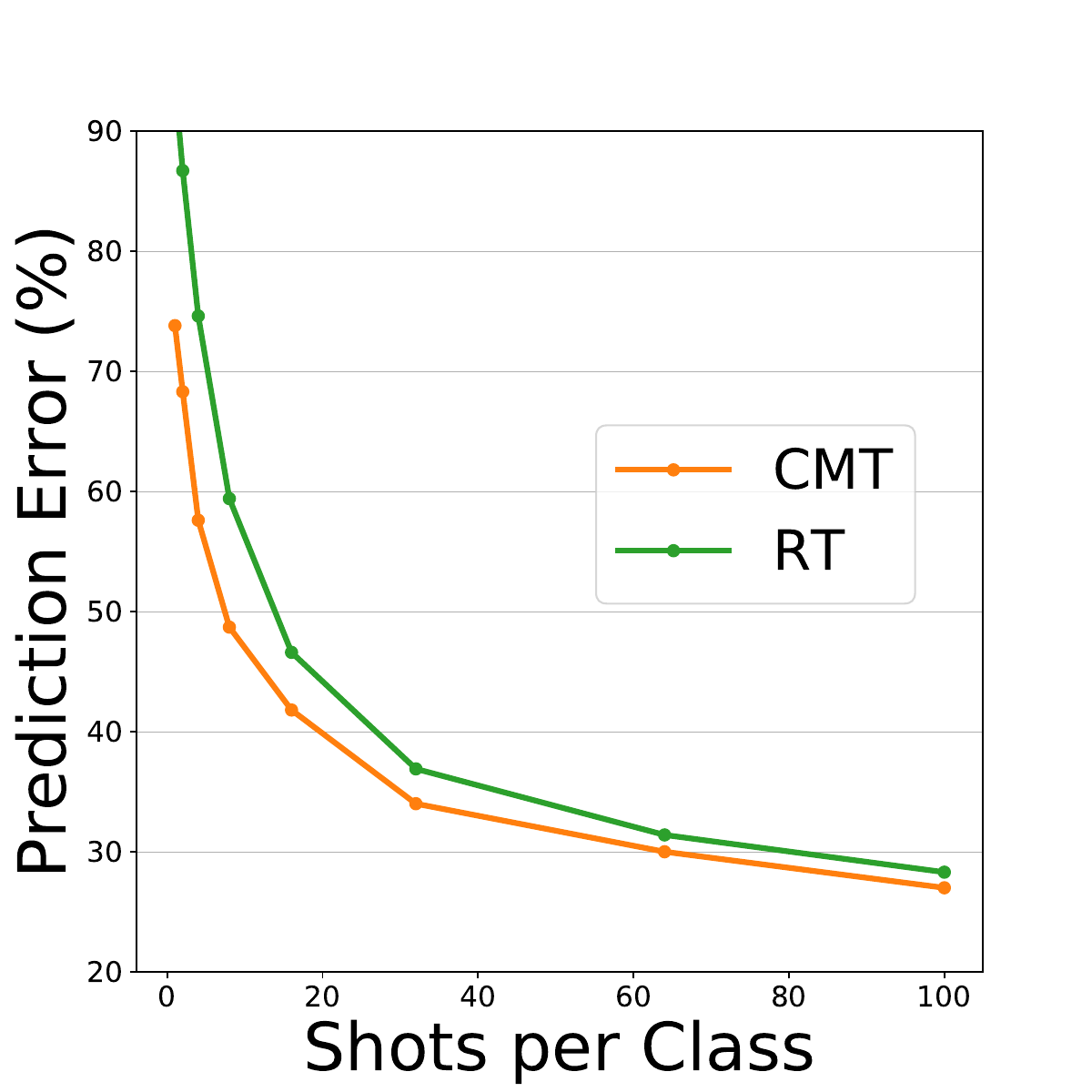}
        \caption{Prediction Error versus Shots}
        \label{fig:aloi_few_shot_perf}
    \end{subfigure}
    \begin{subfigure}[l]{0.35\textwidth}
        \includegraphics[width=1\textwidth,keepaspectratio]{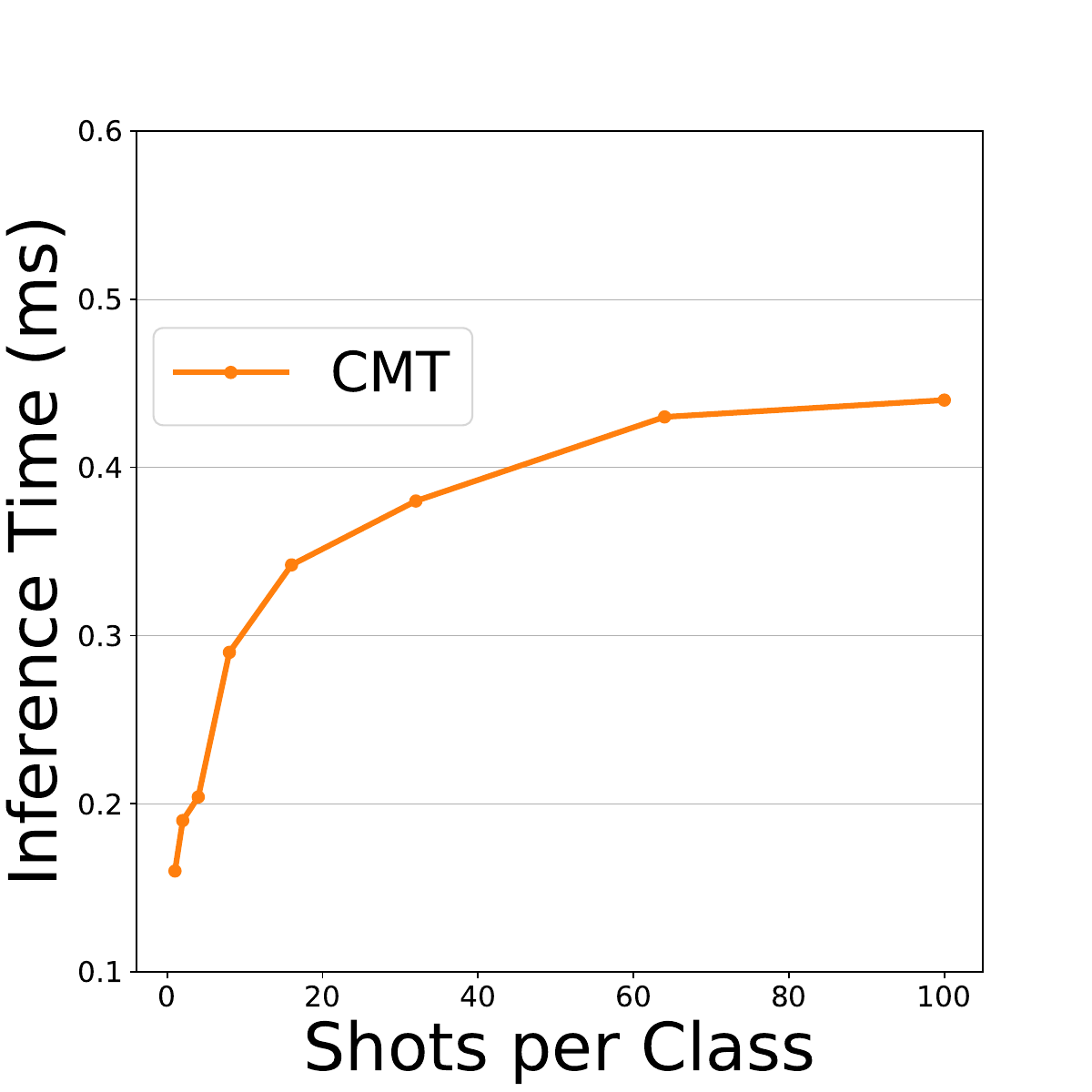}
        \caption{Inference Time verus Shots}
        \label{fig:aloi_few_shot_time}
    \end{subfigure}
    \caption{Performance (a) and inference time (b) of \DCMT{} versus the number of training examples per label in ALOI (i.e., shots $S$).}
\label{fig:aloi_few_shots_details_per_times}
\end{figure}

\autoref{fig:aloi_few_shots_effect_c} shows the detailed plots of \DCMT's statistical performance (a) and inference performance (b) with respect to parameter $c$ (i.e., the maximum number of memories per leaf: $c\log(N)$). As shown in \autoref{fig:aloi_few_shots_effect_c} (b), the inference time increases almost linear with respect to $c$, which is expected as once we reach a leaf, we need to scan all memories stored in that leaf. 

\autoref{fig:aloi_few_shots_details_per_times} shows the detailed plots of \DCMT's statistical performance (a) and inference time (b) with respect to the number of shots (i.e., number of training examples for each class) in ALOI.  Note that ALOI has in total 1000 classes and hence for ALOI $S$-shot, we will have in total $S\times 1000$ examples. Namely as $S$ increases, \DCMT{} has more memories to store. We vary $S$ from 1 to 100. \autoref{fig:aloi_few_shot_perf} shows the performance of \DCMT{} improves quickly as $S$ increases (e.g., dataset becomes easier to learn). Also \DCMT{} consistently outperform Recall Tree, with larger margin at fewer shots. 
From \autoref{fig:aloi_few_shots_details_per_times} (b), we see that the inference time increases sublinearly with respect to the number of shots (i.e., the number of total memories stored in \DCMT), which is also expected, as we show that the depth of \DCMT{} and the number of memories per leaf are logarithmic with respect to the size of \DCMT.

\subsection{Few-shot Extreme Multi-class Classification Details}
\label{sec:few_shot_details}
\begin{table*}[h]
\begin{center}
\resizebox{0.8\textwidth}{!}{ 
    \centering
    \begin{tabular}{llccccccccc}
       \toprule
     & & { \DCMT(u)}   & {\DCMT}   & LOMTree & Recall Tree & OAA &  &   \\ 
      \midrule
    \multirow{ 2}{*}{ALOI} & Test Error  & 75.8  &  {26.3}  & 66.7 & 28.8  & 21.7 &  & \\
    & Test Time & 0.27  & 0.15  & 0.01 & 0.02 & 0.05 &   & \\
    \midrule
    \multirow{ 2}{*}{WikiPara (1-shot)} & Test Error  & 97.3  &  {96.7}  & 98.2 & 97.1  & 98.2 &  & \\
    & Test Time & 0.3  & 0.3  & 0.1 & 0.1 & 0.9 &   & \\
      \rowcolor{black!10}
       & Test Error  & 96.3  &  96.0  & 96.7 & {94.0}  & 95.6 &  &  \\
      \rowcolor{black!10}
    \multirow{-2}{*}{WikiPara (2-shot)}& Test Time & 0.4  & 0.4  & 0.1 & 0.1 & 1.1 &  & \\
       \multirow{ 2}{*}{WikiPara (3-shot)} & Test Error  & 96.1  &  95.7  & 96.1 & {92.0}  & 92.8 &  &  \\
    & Test Time & 0.5  & 0.3  & 0.1 & 0.1 & 1.1 &  & \\
    \midrule
    \multirow{ 2}{*}{ImageNet (1-shot)} & Test Error & 98.8 & 98.7 & 99.8 &  99.7 & 98.0 & &  \\
    & Test Time & 9.6 & 8.2 & 1.0 &  3.3 & 112.4 &  & \\
      \rowcolor{black!10}
     & Test Error & 98.7 & 98.3 & 99.6 &  99.3 & 97.0 &  &  \\
      \rowcolor{black!10}
    \multirow{-2}{*}{ImageNet (2-shot)}& Test Time & 11.7 & 8.6 & 1.2 & 3.3 & 112.0 &  & \\
     \multirow{ 2}{*}{ImageNet (3-shot)} & Test Error & 98.6 & 98.1 & 99.4 &  98.9 & 96.2 &  &  \\
    & Test Time & 9.8 & 8.5 & 4.6 &  3.3 & 109.0  &  & \\
      \rowcolor{black!10}
     & Test Error & 98.4 & 97.9 & 99.2 &  98.6 & 95.3 &  &  \\
      \rowcolor{black!10}
    \multirow{-2}{*}{ImageNet (5-shot)}& Test Time & 12.5 &  11.6  & 1.3 &  4.0 & 110.4 &  & \\
    \bottomrule
    \end{tabular}}
\end{center}
\vspace{-5pt}
\caption{ Prediction error ($\%$) and inference time (ms) of different mult-class classification algorithms on few-shot extreme multi-class classification datasets.}
\label{tab:details_few_shot_table}
\end{table*}

\autoref{tab:details_few_shot_table} shows the detailed prediction error and inference time of \DCMT and other baselines. For ALOI, we briefly tuned the parameters of \DCMT based on a set of holdout training data, and for few-shot WikiPara (and few-shot ImageNet), we briefly tuned the parameters of \DCMT using the one-shot dataset on hold-out dataset and then simply just use the same set of parameters across all other few-shot datasets. The detailed key parameters can be found in \autoref{tab:parameters}. Note that the parameters $c$ (leaf memories multiplier), $d$ (number of \textsc{reroute} calls per insertion), and $\alpha$ (regularization parameter to ensure balance of \DCMT) are the tthree key extra parameters we have compared to the baselines considered here such as Recall Tree and LOMTree. 

One interesting observation from \autoref{tab:details_few_shot_table} is that \DCMT can outperform even OAA at the one-shot WikiPara experiment. All the datasets have same number of examples per class and hence a constant predictor (i.e., prediction by majority) would have prediction accuracy at $1/$(\# of classes). In terms of computation, due to the overhead of storing memory and dynamically allocating memory in \DCMT, \DCMT in general is less computationally efficient than other logarithmic baselines (LOMTree \& Recall Tree). Comparing to highly optimized implementation of OAA in VW, we observe that \DCMT{} is less computationally efficiently on smaller dataset such aas ALOI, while for datasets with extremely large number of labels, \DCMT{} consistently outperform OAA in terms of computation efficiency.


\subsection{Multi-Label Classification}

The key parameters we used to conduct our multi-label experiments are summarized in \autoref{tab:parameters}. We briefly tuned the number of supervised passes and $\alpha$ on holdout training datasets and picked a set of parameters that worked well for all datasets in general. We did not tuned parameters $c$ and $d$.  The results are summarized in table~\ref{tab:multi_label_perf}.

\subsection{Image Retrieval}
\begin{table*}[h]
\begin{center}
\resizebox{0.7\textwidth}{!}{ 
\centering
    \begin{tabular}{llcccccc}
   \toprule
    & & \DCMT (u)   & \DCMT  & NN & KD-Tree w/ PCA     \\ 
    \midrule
     \multirow{ 2}{*}{Pascal } &Test Reward & 0.680$\pm$0.008 & {0.694}$\pm$0.010 & 0.683 $\pm$0.013  &   0.675 $\pm$0.013 \\
    &Test Time (ms) & {0.13} & {0.58} & 0.74 &  0.002 \\
    \midrule
     \multirow{ 2}{*}{Flickr8k} &Test Reward & 0.733$\pm$0.004 & {0.740}$\pm$0.002 & 0.736 $\pm$0.003 & 0.733 $\pm$  0.002 \\
    &Test Time (ms) & {0.23} & {1.0} & 6.0 & 0.002  \\
    \midrule
    \multirow{ 2}{*}{MS COCO} &Test Reward & 0.581 & 0.584 &{0.585}  & 0.574  \\
    &Test Time (ms) & {0.590} &  {1.90} & 12.4 & 35.4  \\

    \bottomrule
    \end{tabular}}
\end{center}
\caption{ Performance (average reward $\%$ and time $ms$) of different approaches on image retrieval tasks.}
\label{tab:image}
\end{table*}

For Pascal and Flickr8k, we randomly split the dataset into a pair of training set and test set. We create 5 random splits, and use one split for tuning parameters for \DCMT. For MS COCO, we use the default training, validation, and test split, and tune parameters on validation set.

For image retrieval applications, the key parameters used by \DCMT{} are summarized in \autoref{tab:parameters}, and the detailed performances of \DCMT and NN are summarized in \autoref{tab:image}.  For Pascal and Flickr8k, since we have 5 training/test split, we report mean and standard deviation.  

In this set of experiments, for \DCMT{}, during training we set $k$ to be $c\log(N)$, i.e., we returned all memories stored in a single leaf to get reward signals to update $f$.  During testing, for both \DCMT{} and NN, we report the average reward of the top returned memory on given test sets. 

\wen{\autoref{tab:image} summarizes the performance of \DCMT{}, NN and KD-Tree operating on a low dimension feature of the query computed from the randomized PCA algorithm from sklearn. We choose the reduced dimension of the feature such that the total PCA time plus the KD-Tree construction time is similar to the time of unsupervised CMT construction time (in Pascal, the reduced dimension is 20; in Flickr8k, the reduced dimension is 200; in MSCOCO, the reduced dimension is 200).  Note that on Pascal and Flickr8k, \DCMT{} slightly outperforms NN in terms of average reward on test sets, indicating the potential benefit of learned memories. \DCMT{} statistically outperforms KD-tree operated on the low dimensional feature computed from PCA.}  


\end{document}